\documentclass{article}

\usepackage[main, final]{neurips_2026}

\usepackage[utf8]{inputenc} 
\usepackage[T1]{fontenc}    
\usepackage{hyperref}       
\usepackage{url}            
\usepackage{booktabs}       
\usepackage{amsfonts}       
\usepackage{nicefrac}       
\usepackage{microtype}      
\usepackage[usenames,dvipsnames]{xcolor}         
\usepackage{microtype}
\usepackage{graphicx}
\usepackage{subcaption}
\usepackage{amsmath}
\usepackage{mathtools}
\usepackage{amsthm}
\usepackage{amssymb}
\usepackage{dsfont}
\usepackage{pifont}
\usepackage{wrapfig}

\theoremstyle{plain}
\newtheorem{theorem}{Theorem}[section]
\newtheorem{proposition}[theorem]{Proposition}

\theoremstyle{definition}

\theoremstyle{remark}

\usepackage{subcaption}
\usepackage{algorithm}
\usepackage{algorithmic}

\title{Differentiable Systematic Resampling for \\ Variational Sequential Monte Carlo}

\author{%
  Fredrik~Cumlin \\
  Dept. of Information Science and Engineering\\
  KTH Royal Institute of Technology\\
  Stockholm, Sweden \\
  \texttt{fcumlin@kth.se} \\
  \And
  Saikat~Chatterjee \\
  Dept. of Information Science and Engineering\\
  KTH Royal Institute of Technology\\
  Stockholm, Sweden \\
  \texttt{sach@kth.se} \\
}

\begin{document}

\maketitle

\begin{abstract}
  Particle filters are a standard tool for nonlinear state estimation, but their resampling step is discrete, preventing gradient-based learning in variational sequential Monte Carlo. We introduce Differentiable Systematic Resampling (DSR), a temperature-controlled relaxation of systematic resampling, that preserves the CDF-ordered, banded structure of systematic resampling while enabling full gradient flow. DSR converges to exact systematic resampling as the temperature vanishes, and we prove a pointwise exponential convergence rate for the induced bias. Compared to optimal-transport-based differentiable resampling, DSR avoids iterative solvers and has substantially lower computational overhead. Experiments on stochastic dynamical systems and real-world handwriting data show that DSR achieves comparable or superior filtering and dynamics learning performance.
\end{abstract}

\section{Introduction}

Sequential Monte Carlo (SMC) methods, commonly known as particle filters (PFs) \cite{gordon1993novel,doucet2001sequential}, are standard tools for Bayesian state estimation in nonlinear and non-Gaussian dynamical systems, with applications in robotics, econometrics, computer vision, and epidemiology \cite{Särkkä_2013,naesseth2019elements,chopin2020introduction}. Particle filters approximate the filtering posterior by propagating and reweighting a set of particles using the system dynamics and incoming observations. In many applications, however, the state dynamics are unknown or only partially specified; that is, the transition model (system dynamics) is not available to the practitioner. One line of work addresses this by learning a state estimator directly from observations without an explicit particle representation \cite{danse, pdanse, vse, dns}; here we instead focus on variational sequential Monte Carlo (VSMC) \cite{naesseth2018variational,maddison2017filtering,le2018auto}, which jointly learns a parameterized transition model and proposal model by maximizing a variational lower bound on the marginal likelihood using gradient-based optimization. Importantly, VSMC requires the entire filtering recursion to be differentiable.

An essential component of SMC, and hence of VSMC, is resampling. During filtering, particle weights tend to degenerate, meaning that only one or a few particles carry most of the weight \cite{tutorial_particle_filter_2011}. Resampling mitigates this by replacing low-weight particles with copies of high-weight ones \cite{douc2005comparison,kitagawa1996monte}. Among standard resampling schemes, systematic resampling is widely adopted for its simplicity and low variance in practice \cite{douc2005comparison,kitagawa1996monte,tutorial_particle_filter_2011}. However, resampling relies on discrete ancestor assignments and is therefore inherently non-differentiable. While particle propagation and importance weighting can be differentiated via the reparameterization trick \cite{reparameterization_trick}, the resampling step blocks gradient flow and produces biased or high-variance gradient estimators \cite{naesseth2018variational}. This makes differentiable resampling a central challenge in VSMC. A standard approach for discrete operations in gradient-based learning is to replace them with temperature-controlled continuous relaxations that recover the discrete operation in the zero-temperature limit~\cite{jang2017categorical,maddison2017concrete}. We take this approach for the resampling step of VSMC, but rather than relaxing the ancestor categorical directly, we relax the assignment of systematic resampling, based on the cumulative distribution function (CDF) of the weights, hence preserving its CDF-ordered, banded structure.

Several differentiable relaxations of resampling have been proposed. Soft multinomial resampling~\cite{karkus2018particle} keeps ancestor sampling discrete but draws ancestors from a smoothed distribution that mixes the importance weights with a uniform distribution, and propagates gradients through the corresponding importance-weight correction. A drawback is that the underlying scheme is multinomial resampling, which has higher variance than residual and stratified resampling and, in practice, than systematic resampling~\cite{douc2005comparison}; hence, the method inherits an inferior starting point. Entropy-regularized optimal transport (EROT)~\cite{corenflos2021differentiable} instead formulates resampling as an optimal transport problem. However, EROT relies on recursive Sinkhorn updates that scale as $\mathcal{O}(N^2)$ per iteration, is sensitive to the choice of regularization parameter, and, at the regularization levels needed for stable training, produces diffuse transport plans that bear little resemblance to the sparse, CDF-ordered structure of systematic resampling (see Fig.~\ref{fig:T_combined}). However, neither approach preserves the structure of systematic resampling, which is the de facto standard at inference time due to its favorable variance properties~\cite{tutorial_particle_filter_2011}.

\begin{figure}
    \centering
    \includegraphics[width=0.7\linewidth]{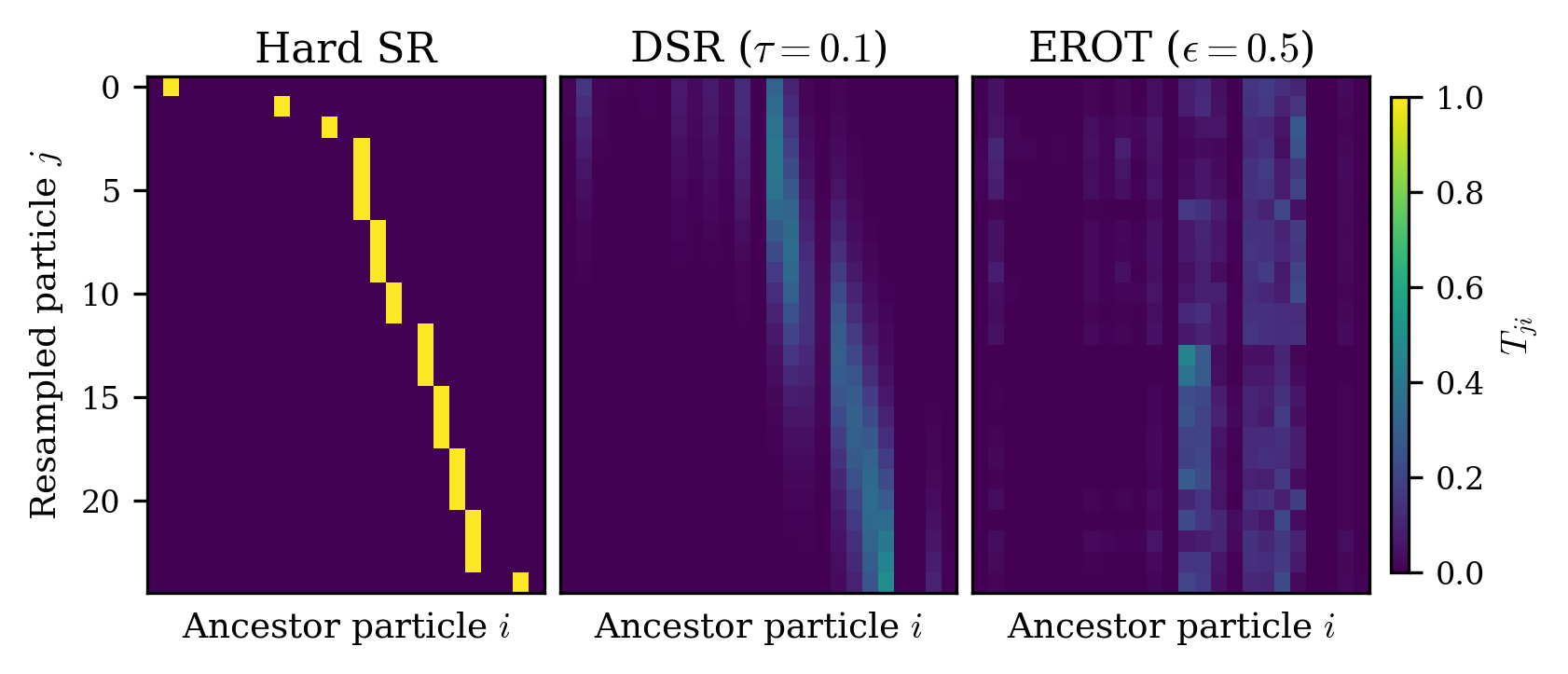}
    \caption{Transportation maps for systematic resampling (left), DSR (middle), and EROT (right), for a sequential importance resampling PF of the Lorenz-63 system. The transportation maps are calculated at step $150$ using $N=25$ particles, and DSR and EROT use the training relaxation parameters.}
    \label{fig:T_combined}
\end{figure}

In this paper, we introduce Differentiable Systematic Resampling (DSR), a temperature-controlled relaxation that directly relaxes the CDF-level ancestor assignment of systematic resampling using sigmoid functions. To further ensure that neighbours in the CDF correspond to spatially nearby particles, we equip DSR with Morton space-filling curve sorting, which reduces approximation error to systematic resampling. The relaxation introduces a bias that we prove vanishes exponentially as the temperature $\tau \to 0^+$. Our main contributions are as follows:
\begin{itemize}
\item We propose DSR, a differentiable relaxation that preserves the CDF-ordered, banded structure of systematic resampling while enabling full gradient flow.
\item We establish theoretical guarantees showing that DSR converges to exact systematic resampling, with a pointwise exponential rate under a positive-margin condition, as $\tau \to 0^+$.
\item We demonstrate that DSR achieves competitive or superior state estimation and dynamics learning in VSMC while training more than $2\times$ faster than EROT.
\end{itemize}
\section{Background}

\subsection{State space models and particle filters}

We consider state space models $\mathbf{x}_t \sim p(\mathbf{x}_t\vert \mathbf{x}_{t-1})$, $\mathbf{y}_t \sim p(\mathbf{y}_t\vert \mathbf{x}_t)$, where $\mathbf{x}_t\in\mathbb{R}^{d_x}$ is the latent state and $\mathbf{y}_t\in\mathbb{R}^{d_y}$ the observation. Particle filters approximate the filtering posterior $p(\mathbf{x}_t\vert \mathbf{y}_{1:t})\approx\sum_{i=1}^N w_t^{(i)}\delta(\mathbf{x}_t-\mathbf{x}_t^{(i)})$ using $N$ weighted particles $\{(\mathbf{x}_t^{(i)}, w_t^{(i)})\}_{i=1}^N$ at time $t$. During inference, particles are propagated through the dynamics, reweighted by the observation likelihood, and resampled to avoid weight degeneracy. Resampling replaces the weighted particle set with an equal-weight set by sampling ancestor indices according to the weights.

\subsection{Variational Sequential Monte Carlo}

Variational Sequential Monte Carlo (VSMC) \cite{naesseth2018variational} provides a framework for jointly learning a transition model $p_{\pmb\phi}(\mathbf{x}_t | \mathbf{x}_{t-1})$ and proposal distribution $q_{\pmb\theta}(\mathbf{x}_t | \mathbf{x}_{t-1}, \mathbf{y}_t)$ by maximizing the so-called surrogate evidence lower-bound (ELBO):
\begin{equation}
\mathcal{L}(\pmb\theta, \pmb\phi) = \mathbb{E}\left[\log \hat{p}(\mathbf{y}_{1:T})\right] = \mathbb{E}\!\left[\sum_{t=1}^T \log\left(\frac{1}{N} \sum_{i=1}^N w_t^{(i)}\right)\right] ,
\end{equation}
where the expectation is over all sampling and resampling operations in the particle filter (the weights before normalization are used). Algorithm~\ref{alg:vsmc} summarizes the VSMC recursion.
 
\begin{algorithm}[t]
\caption{Variational Sequential Monte Carlo (VSMC)}
\label{alg:vsmc}
\begin{algorithmic}[1]
\REQUIRE Initial particles $\{\tilde{\mathbf{x}}_0^{(i)}\}_{i=1}^N$, observations $\{\mathbf{y}_t\}_{t=1}^T$
\FOR{$t = 1$ to $T$}
    \STATE Sample $\mathbf{x}_t^{(i)} \sim q_{\pmb\theta}(\mathbf{x}_t | \tilde{\mathbf{x}}_{t-1}^{(i)}, \mathbf{y}_t)$ for $i=1,\dots,N$
    \STATE Compute weights $w_t^{(i)} =
    \frac{
    p_{\pmb\phi}(\mathbf{x}_t^{(i)} | \tilde{\mathbf{x}}_{t-1}^{(i)})
    p(\mathbf{y}_t | \mathbf{x}_t^{(i)})
    }{
    q_{\pmb\theta}(\mathbf{x}_t^{(i)} | \tilde{\mathbf{x}}_{t-1}^{(i)}, \mathbf{y}_t)
    }$; normalize $\sum_i w_t^{(i)} = 1$
    \STATE $\{\tilde{\mathbf{x}}_t^{(i)}\}_{i=1}^N
    \leftarrow
    \texttt{Resample}\big(\{\mathbf{x}_t^{(i)}, w_t^{(i)}\}_{i=1}^N\big)$ \hfill \textit{(discrete, non-differentiable)}
\ENDFOR
\end{algorithmic}
\end{algorithm}
 
Optimizing $\mathcal{L}$ requires differentiating through the entire filtering recursion. As shown by \cite{naesseth2018variational}, the gradient decomposes into two terms:
\begin{equation}\label{eq:grad_decomp}
\nabla \mathcal{L} = \underbrace{\mathbb{E}\!\left[\nabla \log \hat{p}(\mathbf{y}_{1:T})\right]}_{g_{\mathrm{rep}} \text{ (reparameterization)}} + \underbrace{\mathbb{E}\!\left[\log \hat{p}(\mathbf{y}_{1:T})\; \nabla \log \tilde{\varphi}(\mathbf{a}_{1:T-1}^{1:N} | \pmb\varepsilon_{1:T}^{1:N})\right]}_{g_{\mathrm{score}} \text{ (score function)}},
\end{equation}
where $\tilde{\varphi}$ is the marginal distribution over the discrete ancestor variables $\mathbf{a}_{1:T-1}^{1:N}$ given the reparameterized noise $\pmb\varepsilon_{1:T}^{1:N}$. The reparameterization term $g_{\mathrm{rep}}$ differentiates through the continuous particle propagation and weighting steps via the reparameterization trick~\cite{reparameterization_trick}. The score-function term $g_{\mathrm{score}}$ accounts for the discrete ancestor assignments introduced by resampling, which are not reparameterizable and can have high variance~\cite[cf.\ Fig.\ 3]{naesseth2018variational}.

In practice, $g_{\mathrm{score}}$ is dropped, yielding the biased gradient $\nabla \mathcal{L} \approx g_{\mathrm{rep}}$~\cite{naesseth2018variational}. This treats ancestor indices as fixed during backpropagation: gradients flow through the surviving particle values, but not through the ancestor selection itself, a limitation that becomes severe when weights are concentrated and resampling decisions dominate filtering performance (see Appendix~\ref{app:N_ablation}). Differentiable resampling methods address this by replacing the discrete ancestor selection with a continuous relaxation, so that $g_{\mathrm{rep}}$ flows through the full resampling mechanism. This yields a relaxed objective in place of the surrogate ELBO. At finite temperature the relaxation does not preserve expected replication counts, so the resulting likelihood estimator $\hat{p}(\mathbf{y}_{1:T})$ is no longer unbiased and the relaxed objective is a differentiable surrogate rather than a lower bound on $\log p(\mathbf{y}_{1:T})$ (see Section~\ref{sec:dsr_formulation}). The relaxed objectives have been shown to provide a richer learning signal for the transition model and proposal with improved state estimation and system identification across tasks~\cite{corenflos2021differentiable, vannierop2025deepvsmc, normalizing_flow_dpf, kloss2021train, cdf_placement_dpf}.

\section{Differentiable Systematic Resampling}
\label{sec:dsr}

Traditional resampling methods are discrete and non-differentiable, creating a bottleneck for gradient-based learning. We propose differentiable systematic resampling (DSR): a soft relaxation that converges to systematic resampling in the limit, while allowing gradients to flow through the resampling step. For ease of notation, we omit time subscripts in Sections~\ref{sec:dsr} and~\ref{sec:theory}; particle indices are instead denoted as subscripts.

\subsection{DSR formulation}\label{sec:dsr_formulation}

Recall that systematic resampling computes the CDF $F_i = \sum_{k=1}^i w_k$, draws a single offset $u_0 \sim \mathcal{U}(0, 1/N)$, and assigns $a_j = \min\{i : F_i \geq u_j\}$ where $u_j = u_0 + (j-1)/N$. Equivalently, $a_j$ is the unique index $i$ for which $F_{i-1} < u_j \leq F_i$, so the discrete assignment can be written as a matrix $A_{ij} = \mathds{1}(F_{i-1} < u_j \leq F_i)$.  This indicator is non-differentiable with respect to the weights, blocking gradient flow.

Consider $N = 3$ particles with weights $\mathbf{w} = (0.1, 0.7, 0.2)$, giving the CDF $\mathbf{F} = (0.1, 0.8, 1.0)$ and $F_0 = 0$. Systematic resampling draws a single offset $u_0 \sim \mathcal{U}(0, \tfrac{1}{3})$ and places the comb $u_j = u_0 + (j-1)/3$. For $u_0 = 0.05$ the comb points are $u_1 = 0.05$, $u_2 = 0.38$, $u_3 = 0.72$, and thresholding each against the CDF gives ancestors $a_1 = 1$ (since $F_0 < u_1 \leq F_1$), $a_2 = 2$ (since $F_1 < u_2 \leq F_2$), and $a_3 = 2$ (since $F_1 < u_3 \leq F_2$): the high-weight particle $2$ is replicated, the low-weight particle $3$ is dropped, and the assignment matrix is $A_{ij} = \mathds{1}(F_{i-1} < u_j \leq F_i)$. The only randomness is the single draw $u_0$; given $u_0$, the assignment is a deterministic thresholding of the comb against the CDF. The idea with DSR is to relax only this thresholding step and to leave the draw $u_0$ untouched.

Our strategy is to replace the indicator with a smooth surrogate that recovers it in a limiting sense, following the temperature-relaxation approach common for discrete operations in gradient-based learning~\cite{jang2017categorical,maddison2017concrete}. Writing the indicator as a telescoping difference,
\begin{equation}
    \mathds{1}(F_{i-1} < u \leq F_i) = \mathds{1}(u > F_{i-1}) - \mathds{1}(u > F_i),
\end{equation}
reduces the problem to relaxing each step indicator $\mathds{1}(u > F)$ separately. The sigmoid $\sigma(x) = (1 + e^{-x})^{-1}$ provides a natural smooth surrogate: $\sigma((u-F)/\tau) \to \mathds{1}(u > F)$ pointwise as the temperature $\tau \to 0^+$. Substituting into the telescoping identity with boundary conventions $F_0 = 0$, $F_{N} = 1$ gives the soft assignment:
\begin{equation}\label{eq:soft_s}
    s_i(u; \tau) = \sigma\!\left(\frac{u - F_{i-1}}{\tau}\right) - \sigma\!\left(\frac{u - F_i}{\tau}\right), \quad i=1,\dots, N
\end{equation}
Each $s_i(u;\tau) \in [0, 1]$ is differentiable in $u$, $F_{i-1}$, and $F_i$, and $\sum_i s_i(u;\tau) \to 1$ as $\tau \to 0^+$. At finite $\tau$, the sum $\sum_i s_i(u;\tau)$ is not necessarily one, so we normalize to obtain a proper convex combination. The transportation matrix and resampled particles are:
\begin{equation}\label{eq:soft_assignment}
    T_{ij}^{(\tau)} = \frac{s_i(u_j; \tau)}{\sum_{k=1}^N s_k(u_j; \tau)}, \qquad \tilde{\mathbf{x}}_j^{(\tau)} = \sum_{i=1}^N T_{ij}^{(\tau)} \mathbf{x}_i.
\end{equation}
Each column of $T^{(\tau)}$ sums to one, so $\tilde{\mathbf{x}}_j^{(\tau)}$ is a convex combination of the ancestor particles. The row sums, however, satisfy $\sum_{j=1}^N T^{(\tau)}_{ij} \neq N w_i$ in general, so DSR does not preserve expected replication counts at finite $\tau$. The temperature $\tau$ controls the relaxation between gradient flow and approximation fidelity: small $\tau$ gives a sharp transport plan close to systematic resampling but produces large gradient magnitudes, while large $\tau$ gives a diffuse transport plan with smoother gradients. We show in Section~\ref{sec:theory} that the approximation error of each resampled particle decays exponentially in $1/\tau$, and demonstrate in Section~\ref{sec:experiments} that $\tau = 0.1$ is a robust choice across all our experiments. Forming the full transport matrix $T^{(\tau)} \in \mathbb{R}^{N \times N}$ makes DSR $\mathcal{O}(N^2)$ per resampling step, compared with $\mathcal{O}(N)$ for hard systematic resampling; the resulting overhead is nonetheless small in practice due to batch processing on GPU (Table~\ref{tab:training_time})\footnote{Note that the $\mathcal{O}(N^2)$ cost could be reduced to $\mathcal{O}(NK)$ by keeping,
for each resampled particle $j$, only the $K$ ancestors around its hard ancestor $a_j$,
where the soft assignments concentrate. However, on a GPU the full matrix is cheap, hence, we keep it.}. An algorithmic overview of DSR is given in Appendix~\ref{sec:dsr_algorithm}.

\subsection{Particle ordering}

Since DSR operates on the empirical CDF, the particle ordering determines which particles interact through the soft assignments: particles adjacent in the CDF share mass in the rows of $T^{(\tau)}$. It is therefore desirable to order particles so that CDF neighbours are also spatial neighbours, which keeps the soft mixture more local. Thus, we sort particles along a Morton space-filling curve \cite{morton1966computer} before constructing the CDF; an algorithmic overview of this sorting is provided in Appendix~\ref{sec:dsr_algorithm}. Locality-preserving orderings have been shown to improve convergence properties of systematic resampling~\cite{gerber2019negative}; in DSR, they additionally ensure the soft mixture blends spatially nearby ancestors rather than distant particles (for an analysis of this, see Appendix~\ref{app:empirical_analysis}). The sorting costs $\mathcal{O}(N \log N)$ per resampling step, negligible compared to the overall resampling cost of $\mathcal{O}(N^2)$.

Morton sorting significantly improves the approximation to systematic resampling: at $N = 200$, it reduces the L2 distance between DSR and hard systematic resampled particles by $35\%$ at $\tau=0.01$ and by $15\%$ at $\tau=0.1$ (default used in experiments), compared to unsorted ordering (Appendix~\ref{app:empirical_analysis}). For filtering and transition model quality, Morton sorting offers a modest advantage at low SNR (Appendix~\ref{app:sorting_ablation}).

\section{Theoretical Analysis}\label{sec:theory}

We show that DSR is a consistent relaxation whose bias vanishes exponentially fast in $1/\tau$. Let $\{(\mathbf{x}_i, w_i)\}_{i=1}^N$ be particles with normalized weights $w_i > 0$, $F_i = \sum_{k=1}^i w_k$ the empirical CDF, and $T_{ij}^{(\tau)}$, $\tilde{\mathbf{x}}_j^{(\tau)}$ as defined in Eq.~\eqref{eq:soft_assignment}. The first result establishes that DSR recovers systematic resampling in the limit.

\begin{proposition}[Asymptotic exactness]\label{prop:conv_gksr_to_sr}
Fix $j \in [N]$ and let $a^*_j = \min\{i \in [N] : F_i \ge u_j\}$ denote the hard systematic resampling index. Then
$\lim_{\tau \to 0^+} T^{(\tau)}_{ij} = \mathds{1}(i = a^*_j)$ and
$\lim_{\tau \to 0^+} \tilde{\mathbf{x}}^{(\tau)}_j = \mathbf{x}_{a^*_j}$ almost surely.
\end{proposition}

The second result quantifies the approximation error for finite $\tau$.

\begin{proposition}[Exponential convergence rate]\label{prop:exponential_convergence}
Let $D = \sup_{i,j}\|\mathbf{x}_i-\mathbf{x}_j\|_2$ denote the particle diameter and
\begin{equation}
\Delta_j := \min\{u_j - F_{a^*_j-1},\; F_{a^*_j} - u_j\} > 0 \ \text{a.s.}
\end{equation}
the margin between $u_j$ and the nearest CDF boundary. Then there exists $\tau_0 > 0$ such that for all $\tau \leq \tau_0$,
\begin{equation}\label{eq:exp_bound}
\|\tilde{\mathbf{x}}^{(\tau)}_j - \mathbf{x}_{a^*_j}\|_2 \leq 2DN \exp\!\left(-\frac{\Delta_j}{\tau}\right).
\end{equation}
\end{proposition}

The proofs are in Appendix~\ref{sec:proof_of_props}.

Proposition~\ref{prop:exponential_convergence} is primarily a qualitative guarantee: the bound depends on the particle diameter $D$, the count $N$, and the margin $\Delta_j$ between $u_j$ and the nearest CDF boundary. The dominant term is $\exp(-\Delta_j/\tau)$, so concentrated weights (large $\Delta_j$) yield fast convergence while near-uniform weights, where $\Delta_j$ scales as $\mathcal{O}(1/N)$, bound the worst-case rate. The bound holds pointwise in $j$ under the positive-margin event $\Delta_j > 0$ (which holds almost surely); it is not uniform over offsets, since $\Delta_j$ can be made arbitrarily small by the weight and offset configuration, so no uniform exponential rate holds. Empirically, the per-particle L2 distance between DSR and hard systematic resampling decreases as $\tau \to 0^+$, from $0.36$ at $\tau=0.1$ to $0.10$ at $\tau=0.01$ for $N=25$ particles (Appendix~\ref{app:empirical_analysis}).

\section{Experiments}\label{sec:experiments}

We evaluate DSR on synthetic and real-world state estimation and dynamics learning tasks. The experiments assess four aspects: (1) training-time filtering posterior quality $\hat{p}(\mathbf{x}_t| \mathbf{y}_{1:t})$; (2) state estimation accuracy at inference time in standard SMC; (3) transition model quality $\hat{p}(\mathbf{x}_t|\mathbf{x}_{t-1})$; and (4) computational efficiency.

\subsection{Experimental Setup}

We consider nonlinear state-space models where the transition and proposal models are parameterized by fully connected neural networks and learned jointly under the VSMC objective. All methods use $N = 25$ particles unless stated otherwise (Appendix~\ref{app:N_ablation} reports a sweep over $N \in \{25, 50, 100, 200\}$), the same Adam~\cite{adam} configuration, identical network architectures, and gradient clipping at norm $5.0$. We clip the estimated log-likelihood below to avoid extreme instability in training, see Appendix~\ref{sec:training_arch_config}. For DSR and EROT, the gradient with respect to the transport plan is clipped elementwise to $[-1, 1]$, as in the implementation \cite{corenflos2021differentiable}. Results are averaged over $10$ independent runs ($3$ for the linear Gaussian SSM); we report mean $\pm$ standard deviation. Training configuration can be found in the Appendix~\ref{sec:training_arch_config}.

We evaluate on three settings: (1) a 4-dimensional \textit{Linear Gaussian SSM}, where the Kalman filter gives the optimal posterior; (2) \textit{Lorenz-63}, a 3D chaotic attractor with a nonlinear camera observation map; and (3) \textit{CharacterTrajectories}, real-world handwriting dynamics with 3D position tracking. We use different signal-to-measurement noise (SMNR) ratios to experiment with different noise levels (definition of SMNR can be found in Appendix~\ref{app:SMNR}).

We report NMSE (dB) for state estimation, one-step MSE and five-step MSE for transition model quality, log-likelihood (LL) and Kullback--Leibler divergence (KLD) for distribution matching, and wall-clock training time. Full definitions are in Appendix~\ref{sec:evaluation_measures}.

We compare DSR to three resampling baselines: \textit{Hard}, standard systematic resampling~\cite{jonschkowski2018differentiable}, which corresponds to estimating the ELBO gradient via $g_{\mathrm{rep}}$ in Eq.~\eqref{eq:grad_decomp}; \textit{Soft}, relaxation of multinomial resampling with $\alpha = 0.5$~\cite{karkus2018particle}; and \textit{EROT}, entropy-regularized optimal-transport resampling with $\varepsilon = 0.5$ and at most $100$ Sinkhorn iterations, with early stopping~\cite{corenflos2021differentiable}. All methods produce biased estimates of the surrogate ELBO gradient. Our code is available at \url{https://github.com/fcumlin/differentiable-systematic-resampling.git} and \href{https://github.com/SachLab/DSR/}{https://github.com/SachLab/DSR/}.

\subsection{Linear Gaussian State Space Model}\label{sec:linear_gaussian_experiment}

To evaluate DSR in a controlled setting where the ground truth is analytically available, we consider a 4-dimensional linear Gaussian state space model
\begin{equation}
    \mathbf{x}_t = \mathbf{A} \mathbf{x}_{t-1} + \mathbf{e}_t, \quad \mathbf{y}_t = \mathbf{x}_t + \mathbf{w}_t,
\end{equation}
with $\mathbf{e}_t \sim \mathcal{N}(\mathbf{0}, \sigma_e^2 \mathbf{I}_4)$, $\mathbf{w}_t \sim \mathcal{N}(\mathbf{0}, \sigma_w^2 \mathbf{I}_4)$, process noise at $-10$ dB and observation noise set by the SMNR. The true transition matrix is a damped block-diagonal rotation $\mathbf{A} = 0.9 \cdot \mathrm{blockdiag}(\mathbf{R}(\theta_1), \mathbf{R}(\theta_2))$ with $\theta_1 = 0.3$, $\theta_2 = 0.5$ and $\mathbf{R}(\theta)$ the $2\times2$ rotation matrix. Since the model is linear Gaussian, the Kalman filter gives the optimal posterior at each $t$, which lets us compare the training-time particle distribution against a Bayesian oracle via KLD. We parameterize the transition model as $\hat{\mathbf{A}} \in \mathbb{R}^{4\times4}$ with a learnable diagonal covariance. We train with $N = 25$ particles on $1000$ trajectories of length $T = 100$ at SMNR $= 10$ dB.

\begin{table}[ht]
    \centering
    \caption{Performance on the linear Gaussian SSM ($d_x = 4$, SMNR $= 10$ dB). KLD KF is the training-time divergence to the Kalman posterior; NMSE is evaluated with $N = 1000$ particles and hard systematic resampling; $\|A - \hat{A}\|_F$ measures recovery of the true transition matrix. Results show mean $\pm$ std over $3$ runs.}
    \label{tab:linear_ssm}
    \begin{tabular}{lccc}
        \toprule
        Method & KLD KF & NMSE (dB) & $\|A - \hat{A}\|_F$ \\
        \midrule
        Hard & $\mathbf{8.64} {\scriptstyle\pm0.02}$ & $\underline{-11.33} {\scriptstyle\pm 0.06}$ & $0.035 {\scriptstyle\pm 0.001}$ \\
        Soft & $\underline{9.72} {\scriptstyle\pm 0.05}$ & $-10.33 {\scriptstyle\pm 0.07}$ & $0.025 {\scriptstyle\pm 0.002}$ \\
        EROT & $12.55 {\scriptstyle\pm 0.03}$ & $-11.31 {\scriptstyle\pm 0.02}$ & $\underline{0.022} {\scriptstyle\pm 0.004}$ \\
        DSR  & $12.30 {\scriptstyle\pm 0.03}$ & $\mathbf{-11.34} {\scriptstyle\pm 0.01}$ & $\mathbf{0.019} {\scriptstyle\pm 0.001}$ \\
        \bottomrule
    \end{tabular}
\end{table}

Hard achieves the best training-time posterior (KLD $8.64$ vs.\ $12.30$ for DSR), yet DSR recovers the true transition matrix most accurately ($\|A - \hat{A}\|_F$ of $0.019$ vs.\ $0.035$). A better training-time posterior therefore does not imply a better gradient signal for parameter learning, which is consistent with the gradient-norm analysis on Lorenz-63 (Section~\ref{sec:gradient_norm}), where DSR's smoother gradients lead to better downstream performance despite comparable point-wise posterior quality. Appendix~\ref{app:linear_gaussian_extra} reports two further linear Gaussian studies: an evaluation-time decomposition of the marginal likelihood, and a comparison of the DSR gradient against an unbiased VSMC gradient.

\subsection{Lorenz-63}

We evaluate all methods on the (stochastic) Lorenz-63 attractor with a nonlinear camera measurement map, following \cite{vannierop2025deepvsmc}. The Lorenz-63 system is a 3-dimensional chaotic attractor given by the discretized dynamics
\begin{equation}
\mathbf{x}_{t+1} = F(\mathbf{x}_t)\mathbf{x}_t + \mathbf{e}_t, \quad \mathbf{e}_t \sim \mathcal{N}(0, \sigma_e^2 I),
\end{equation}
where $F(\mathbf{x}_t) = \exp\left(\begin{bmatrix} -10 & 10 & 0 \\ 28 & -1 & -x_{t,1} \\ 0 & x_{t,1} & -8/3 \end{bmatrix} \Delta\tau\right)$, with $\Delta\tau = 0.02$ and $\sigma_e^2 = 0.1$ corresponding to $-10$ dB process noise. The values follow \cite{danse, pdanse}. We use a 2D camera model with Gaussian point spread functions over an $8 \times 8$ pixel grid, following \cite{buchnik2024latentkalman, vannierop2025deepvsmc}. The grid points $\{g_i\}_{i=1}^{64}$ are uniformly distributed in the rectangle $[-30, 30] \times [-40, 40]$. The state is projected into the camera plane according to its first two coordinates, with the third coordinate representing depth. The measurement function $\mathbf{h}: \mathbb{R}^3 \to \mathbb{R}^{64}$ is given by
\begin{equation}\label{eq:lorenz_obs_map}
\mathbf{h}(\mathbf{x}_t)_i = 10 \exp\left(-\frac{0.5}{7 + x_{t,3}} \left\|g_i - \begin{bmatrix} x_{t,1} \\ x_{t,2} \end{bmatrix}\right\|_2^2\right),
\end{equation}
where $\mathbf{h}(\mathbf{x}_t)_i$ represents the intensity at pixel $i$. The observations are given by $\mathbf{y}_t = \mathbf{h}(\mathbf{x}_t) + \mathbf{w}_t$ with $\mathbf{w}_t \sim \mathcal{N}(\mathbf{0}, \sigma_w^2 \mathbf{I})$, where $\sigma_w^2$ is determined by the signal-to-measurement noise ratio (SMNR). For training, we use $1000$ trajectories of observations $\mathbf{y}_{1:T}$ with $T=100$; we experiment with two SMNR levels: $10$ dB and $0$ dB. A visualization of the Lorenz-63 system under the camera mapping is shown in Fig.~\ref{fig:lorenz_camera}.

We evaluate both state estimation and transition model quality (Table~\ref{tab:lorenz63_combined}). State estimation uses $10$ held-out trajectories with $T=1000$ and the systematic resampling at inference; the transition model $f_\theta$ is evaluated directly on held-out trajectories.

\begin{figure*}[t]
    \centering
    \includegraphics[width=0.9\textwidth]{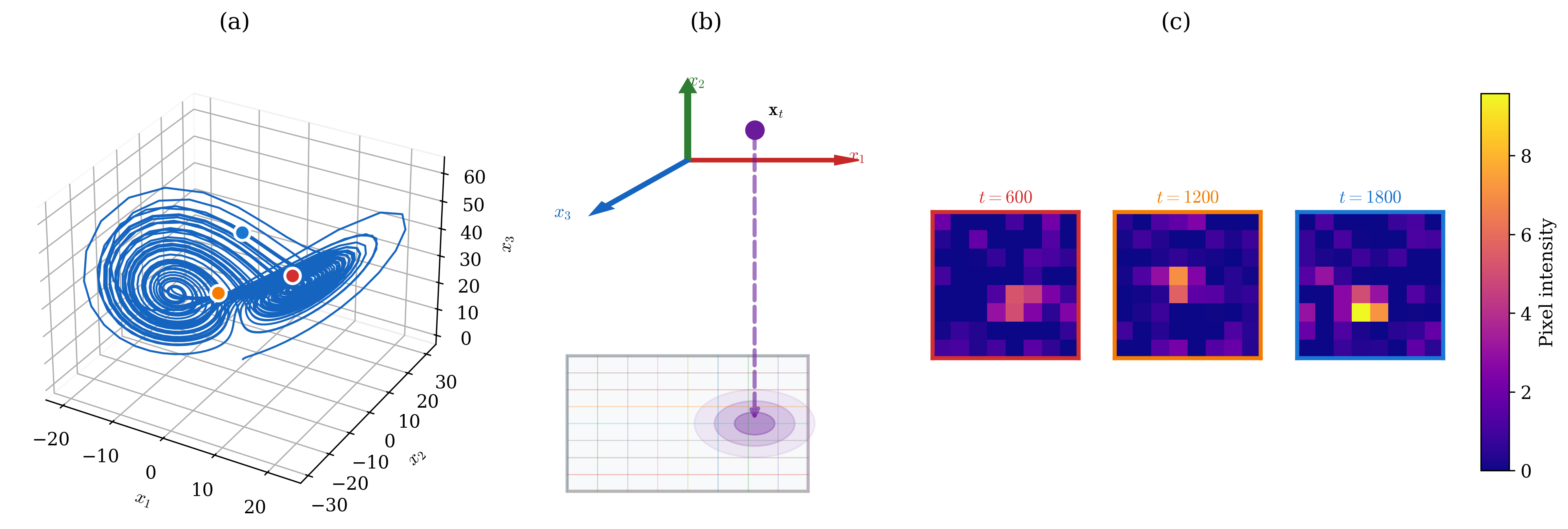}
    \caption{\textbf{Lorenz-63 system with nonlinear camera measurement model.} 
    (a) Sample state trajectory of Lorenz-63. Colored points indicate three observation times that are shown in the right-most figure. 
    (b) Camera projection geometry: the 3D state $\mathbf{x}_t = (x_1, x_2, x_3)$ projects onto an 8×8 pixel grid via a depth-dependent Gaussian point spread function (Eq.~\ref{eq:lorenz_obs_map}). The first two coordinates determine the 2D position, while the third coordinate (depth) controls the spread of the Gaussian.
    (c) Three sample camera observations at SMNR = 0~dB.}
    \label{fig:lorenz_camera}
\end{figure*}

\begin{table}[t]
\centering
\caption{Performance comparison on the Lorenz-63 system at SMNR = 10 and 0 dB. Results show mean $\pm$ std over 10 runs. Best in \textbf{bold}, second-best \underline{underlined}.}
\label{tab:lorenz63_combined}
\setlength{\tabcolsep}{3pt}
\small
\begin{tabular}{lccccc}
\toprule
& \multicolumn{1}{c}{State Estimation} & \multicolumn{4}{c}{Transition Model Quality} \\
\cmidrule(lr){2-2} \cmidrule(lr){3-6}
Method & NMSE (dB) & MSE & MSE-5-step & LL & KLD \\
\midrule
\multicolumn{6}{c}{SMNR = $10$ dB} \\
\midrule
Hard & $\underline{-28.86}{\scriptstyle\pm 0.41}$ & $\underline{0.55}{\scriptstyle\pm 0.19}$ & $5.50{\scriptstyle\pm 3.68}$ & $\underline{-2.26}{\scriptstyle\pm 0.25}$ & $\underline{5.41}{\scriptstyle\pm 1.47}$ \\
Soft & $-27.82{\scriptstyle\pm 0.45}$ & $0.86{\scriptstyle\pm 0.49}$ & $6.45{\scriptstyle\pm 3.92}$ & $-2.55{\scriptstyle\pm 0.48}$ & $6.90{\scriptstyle\pm 2.43}$ \\
EROT & $-28.76{\scriptstyle\pm 0.31}$ & $0.56{\scriptstyle\pm 0.24}$ & $\underline{3.39}{\scriptstyle\pm 2.59}$ & $-2.32{\scriptstyle\pm 0.20}$ & $5.74{\scriptstyle\pm 1.03}$ \\
DSR  & $\mathbf{-29.94}{\scriptstyle\pm 0.28}$ & $\mathbf{0.47}{\scriptstyle\pm 0.24}$ & $\mathbf{3.38}{\scriptstyle\pm 2.11}$ & $\mathbf{-1.98}{\scriptstyle\pm 0.27}$ & $\mathbf{3.60}{\scriptstyle\pm 1.04}$ \\
\midrule
\multicolumn{6}{c}{SMNR = $0$ dB} \\
\midrule
Hard & $-11.73{\scriptstyle\pm 1.75}$ & $0.96{\scriptstyle\pm 0.88}$ & $9.64{\scriptstyle\pm 13.71}$ & $\mathbf{-3.04}{\scriptstyle\pm 0.40}$ & $\mathbf{12.40}{\scriptstyle\pm 3.95}$ \\
Soft & $-11.69{\scriptstyle\pm 1.65}$ & $\underline{0.92}{\scriptstyle\pm 0.08}$ & $\underline{8.19}{\scriptstyle\pm 6.47}$ & $-3.17{\scriptstyle\pm 0.14}$ & $13.87{\scriptstyle\pm 2.32}$ \\
EROT & $\underline{-12.57}{\scriptstyle\pm 1.64}$ & $4.79{\scriptstyle\pm 5.65}$ & $48.41{\scriptstyle\pm 60.72}$ & $-4.28{\scriptstyle\pm 1.54}$ & $33.82{\scriptstyle\pm 29.79}$ \\
DSR  & $\mathbf{-18.28}{\scriptstyle\pm 0.46}$ & $\mathbf{0.54}{\scriptstyle\pm 0.33}$ & $\mathbf{3.70}{\scriptstyle\pm 2.87}$ & $\underline{-3.06}{\scriptstyle\pm 0.24}$ & $\underline{13.41}{\scriptstyle\pm 1.94}$ \\
\bottomrule
\end{tabular}
\end{table}
At $10$ dB SMNR, DSR achieves the best NMSE ($-29.94$ dB) with low run-to-run variance, and the best MSE, LL, and KLD on the transition model, matching EROT on multi-step prediction. Hard and Soft struggle more on longer horizons.

The gap widens at $0$ dB SMNR, where DSR improves NMSE by $5.7$ dB over the second-best method (EROT at $-12.57$ dB). At low SNR, particle weights become highly concentrated and amplify noisy gradients; DSR's smaller and steadier gradients (Table~\ref{tab:gradient_stats}) yield more stable parameter learning. The resulting transition model is not best on every one-step metric (LL, KLD), but produces particles that filter more effectively at inference time, since one-step metrics from true states do not capture error compounding during sequential filtering.

The large MSE-5-step standard deviation of EROT at $0$ dB ($60.72$) raises the question of whether this reflects the method or the choice of $\varepsilon$. We therefore evaluate EROT at $\varepsilon = 1$ as well (Appendix~\ref{app:erot_eps}). Increasing $\varepsilon$ stabilizes training (MSE-5-step std drops to $7.58$) and improves NMSE to $-13.47$ dB, but DSR still outperforms EROT on every metric. 

\subsubsection{Gradient norm}\label{sec:gradient_norm}

\begin{figure}[t]
\centering
\includegraphics[width=0.6\textwidth]{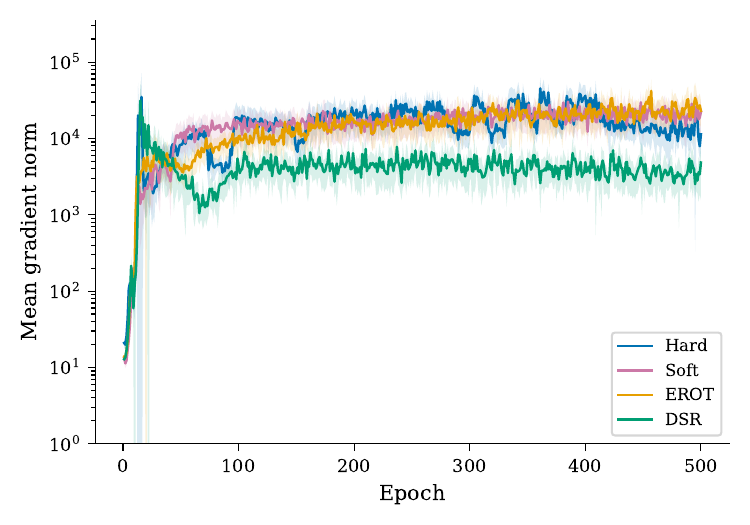}
\caption{Norm of the proposal network's gradient before clipping during training on Lorenz-63 (SMNR $=10$~dB). Solid: per-epoch mean over batches, averaged over the $10$ runs; shaded: $\pm 1$ std over runs.}
\label{fig:gradient_norms}
\end{figure}

\begin{wraptable}{r}{0.42\textwidth}
\centering
\vspace{-\baselineskip}
\caption{Proposal gradient norm before clipping on Lorenz-63 (SMNR $= 10$ dB), averaged over runs and epochs.}
\label{tab:gradient_stats}
\small
\begin{tabular}{lcc}
\toprule
Method & Mean & Within-epoch std \\
\midrule
Hard & $17399$ & $17253$ \\
Soft & $17013$ & $8314$ \\
EROT & $15771$ & $6918$ \\
DSR  & $\mathbf{4504}$ & $\mathbf{2720}$ \\
\bottomrule
\end{tabular}
\end{wraptable}

To assess training stability, we study the norm of the proposal network's gradient, before clipping, under each resampling method on Lorenz-63 at SMNR $= 10$~dB. Figure~\ref{fig:gradient_norms} shows the mean gradient norm per epoch. DSR maintains gradient norms around $4 \times 10^3$ throughout training, whereas Hard and EROT exhibit norms in the $10^4$--$10^5$ range with substantially larger batch-to-batch fluctuations. Table~\ref{tab:gradient_stats} confirms this across runs: DSR produces gradient norms roughly $3\times$ lower than Hard and EROT, with a $2.5-6\times$ lower within-epoch standard deviation. Lower gradient norms do not in themselves imply a better gradient signal, but combined with the improved performance of DSR, the pattern indicates more stable learning.

\subsubsection{Scaling with particle count}
\label{sec:lorenz_N}

To assess how each method scales with the number of particles, we train on Lorenz-63 at SMNR $= 10$~dB with $N \in \{25, 50, 100, 200\}$, keeping all other hyperparameters fixed. DSR is the only method that trains stably at every tested $N$: its NMSE improves from $-29.94$~dB at $N = 25$ to $-30.35$~dB at $N = 200$ with run-to-run standard deviation below $0.3$~dB throughout. Hard and Soft are competitive at $N \leq 100$ but collapse at $N = 200$, with NMSE degrading to $-20.42$ and $-21.52$~dB and standard deviations above $7$ and $12$~dB, respectively, which is caused by training instability. EROT does not complete at $N = 200$: the $\mathcal{O}(N^2)$ transport matrix combined with at most $100$ Sinkhorn iterations per resampling step and the autodiff tape over $T = 100$ time steps exceeds available GPU memory. Results and a discussion are reported in Appendix~\ref{app:N_ablation}.

\subsection{Character typing}

To demonstrate applicability to real-world data, we evaluate on the CharacterTrajectories dataset \cite{williams2006character} from the UEA multivariate time series archive~\citep{bagnall2018uea}, $2,858$ pen-tip recordings of $20$ lower-case characters captured on a WACOM tablet at $200$ Hz.

\begin{figure}[t]
    \centering
    \includegraphics[width=0.7\columnwidth]{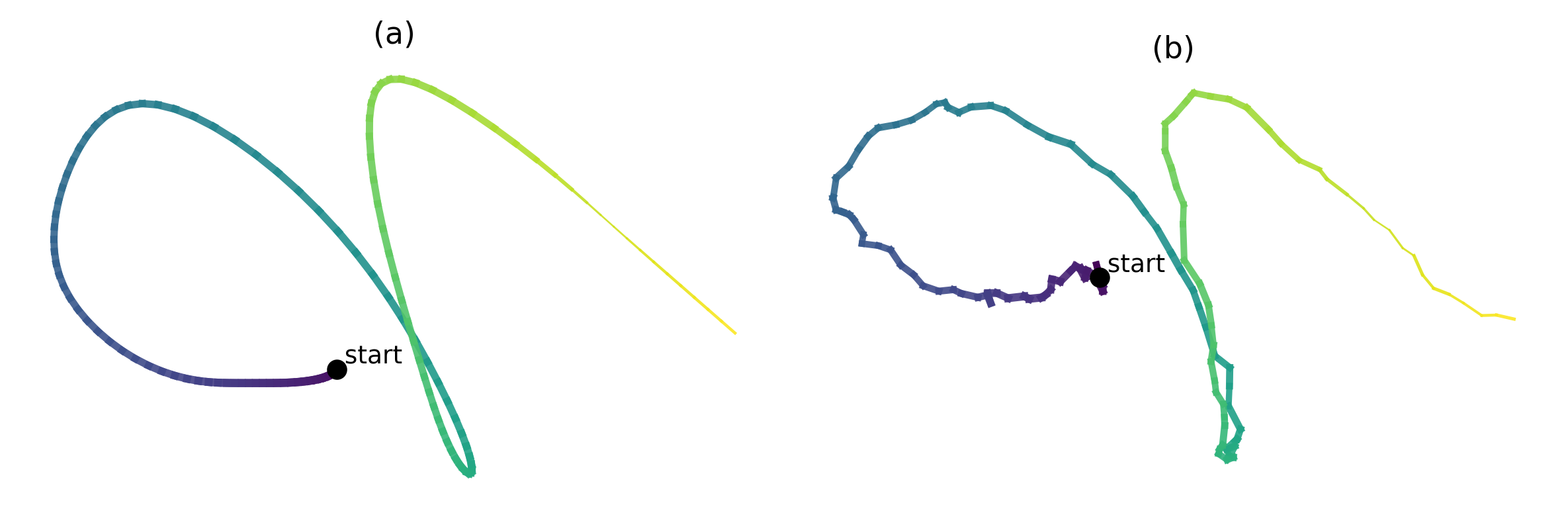}
    \caption{Handwriting dynamics from the CharacterTrajectories dataset. Sample trajectory of the letter `a' captured at 200~Hz on a WACOM tablet. (a) pen-tip path obtained by integrating the recorded x and y velocities; line thickness shows pen force and colour shows time from start (black dot) to finish. (b) the same after adding noise to the recorded channels at SMNR $= 10$ dB.}
    \label{fig:character_trajectory}
\end{figure}

The raw observations $\mathbf{y}_t \in \mathbb{R}^3$ are the pen-tip $x$ and $y$ velocities and the pen force. To give the learned dynamics room to encode temporal information not directly observed, we augment the state to $\mathbf{x}_t = [\mathbf{s}_t^{(1)}, \mathbf{s}_t^{(2)}, \mathbf{s}_t^{(3)}]^\top \in \mathbb{R}^9$ with three 3D latent components, and observe only the last:
\begin{equation}
\mathbf{y}_t = \begin{bmatrix} 0_3 & 0_3 & I_3 \end{bmatrix} \mathbf{x}_t + \mathbf{w}_t, \qquad \mathbf{w}_t \sim \mathcal{N}(0, \mathrm{diag}(\sigma_{w,1}^2, \sigma_{w,2}^2, \sigma_{w,3}^2)),
\end{equation}
each $\sigma_{w,k}^2$ set by the SMNR of channel $k$ in that recording. The remaining six dimensions are latent and learned end-to-end. We use windows of $178$ samples. All recordings of a character have the same length in this release ($109$–$205$ samples), so only six characters are long enough (a, m, o, q, w, z): we train on the $526$ such recordings among the first $2,000$ and report NMSE on the next $10$, all of the character m, with $N = 25$ particles for training and $N = 1000$ for evaluation. A sample trajectory is shown in Fig.~\ref{fig:character_trajectory}.

\begin{table}[htbp]
\centering
\caption{State-estimation NMSE on CharacterTrajectories ($10$ held-out recordings of the character m; $N = 1000$, hard systematic resampling). Results show mean $\pm$ std over 10 runs.}
\label{tab:performance_character}
\begin{tabular}{lcc}
\toprule
& \multicolumn{2}{c}{\textbf{NMSE (dB)}} \\
\cmidrule(lr){2-3}
Method & SMNR = 10 dB & SMNR = 0 dB \\
\midrule
Hard & $-14.65{\scriptstyle\pm 0.20}$ & $-7.66{\scriptstyle\pm 0.35}$ \\
Soft & $-14.69{\scriptstyle\pm 0.29}$ & $-7.08{\scriptstyle\pm 1.12}$ \\
EROT & $\underline{-14.88}{\scriptstyle\pm 0.17}$ & $\underline{-7.77}{\scriptstyle\pm 1.26}$ \\
DSR & $\mathbf{-15.21}{\scriptstyle\pm 0.15}$ & $\mathbf{-8.59}{\scriptstyle\pm 0.22}$ \\
\bottomrule
\end{tabular}
\end{table}

The results can be found in Table~\ref{tab:performance_character}. At 10 dB SMNR, all differentiable methods achieve comparable performance, with DSR obtaining marginally better NMSE ($-15.21$ dB) than EROT ($-14.88$ dB) and Soft ($-14.69$ dB). The non-differentiable Hard baseline ($-14.65$ dB) is within one standard deviation of Soft, so at SMNR $10$ dB, gradient flow through resampling gives at most a small gain here ($0.56$ dB for DSR over Hard). However, we can observe a larger difference at $0$ dB SMNR. DSR achieves $-8.59$ dB NMSE, which is a $0.82$ dB improvement over EROT ($-7.77$ dB), which corresponds to a $17\%$ lower MSE.

\subsection{Training time comparison}

\begin{table}
\centering
\vspace{-\baselineskip}
\caption{Training time on Lorenz-63.}
\label{tab:training_time}
\small
\begin{tabular}{lcc}
\toprule
Method & Total (h) & Time/epoch (s) \\
\midrule
Hard & $\mathbf{2.49}{\scriptstyle\pm 0.15}$ & $\mathbf{17.94}{\scriptstyle\pm 1.11}$ \\
Soft & $4.33{\scriptstyle\pm 0.05}$ & $31.16{\scriptstyle\pm 0.39}$ \\
EROT & $5.96{\scriptstyle\pm 0.11}$ & $42.94{\scriptstyle\pm 0.77}$ \\
DSR  & $\underline{2.55}{\scriptstyle\pm 0.02}$ & $\underline{18.39}{\scriptstyle\pm 0.16}$ \\
\bottomrule
\end{tabular}
\end{table}
Table~\ref{tab:training_time} reports wall-clock training times on Lorenz-63, using NVIDIA T4 GPU. DSR adds only $2.4\%$ over Hard (2.55 h vs. 2.49 h), while EROT is more than 2× slower (5.96 h) because of its Sinkhorn iterations. Although DSR shares the $\mathcal{O}(N^2)$ transport matrix cost with EROT, it avoids the iterative solver: the soft assignment in Eq.~\eqref{eq:soft_assignment} is a single pass of sigmoid evaluations, implemented as batched matrix operations, hence the nested loops in Algorithm~\ref{alg:dsr} do not appear in practice. This brings DSR's per-resampling cost close to that of a discrete systematic resampling step.

\paragraph{The convex combination limitation.}
At any finite $\tau$, DSR produces convex combinations of ancestor particles rather than exact copies, which may place mass between modes when the posterior is multimodal; the same property applies to EROT~\cite{corenflos2021differentiable}. DSR's CDF-ordered structure partially mitigates this: Morton sorting makes soft assignments mix mostly nearby particles, which limits cross-mode mixing when the modes differ in the leading coordinates of the Morton key. We quantify this on a static bimodal Gaussian in Appendix~\ref{app:multimodal} (separating distributions along first dimension): without Morton sorting every resampled particle lands between the modes, whereas with Morton sorting the inter-modal mass drops to $\approx 0.28$ and remains stable up to dimension $10$, while EROT places most mass between the modes for $d \geq 3$. Nonetheless, multimodal posteriors remain challenging for particle filters in general~\cite{vermaak2023mixturepf,thrun2001robustmc}, and a fuller treatment of differentiable resampling under such posteriors is an interesting direction for future work.

\section{Conclusion}

In this paper we have introduced Differentiable Systematic Resampling (DSR), a principled and computationally efficient method for enabling gradient-based learning in particle filters. DSR is a soft relaxation of systematic resampling, where a parameter is introduced that controls the tradeoff between gradient flow and exactness of systematic resampling. Unlike existing differentiable resampling methods, DSR preserves the CDF-ordered, banded structure of systematic resampling while enabling full gradient flow through the particle filtering recursion. Our experimental results on both synthetic chaotic systems (Lorenz-63) and real-world handwriting dynamics (CharacterTrajectories) demonstrate that DSR achieves the best state estimation accuracy in every setting we tested and the best or near-best dynamics learning quality, with the largest gains at low signal-to-noise ratios.

\section{Acknowledgments}
The research is supported by funding and activities from \href{https://www.digitalfutures.kth.se/}{Digital Futures Center}, \href{https://defence-industry-space.ec.europa.eu/system/files/2023-06/REACTII-Factsheet_EDF22.pdf}{European Defence Fund REACT II} project, \href{https://www.saab.com/}{Saab}, \href{https://www.ericsson.com/en}{Ericsson}, \href{https://www.vinnova.se/}{VINNOVA} and partially supported by the Wallenberg AI, Autonomous Systems and
Software Program (\href{https://wasp-sweden.org/}{WASP}). Computational resources were provided by the National Academic Infrastructure for Supercomputing in Sweden (NAISS), funded by the Swedish Research Council.



\bibliographystyle{plain}  
\small
\bibliography{Reference}
\normalsize

\newpage
\appendix

\section{Differentiable Systematic Resampling Algorithm}\label{sec:dsr_algorithm}

In this section, we explain the DSR algorithm. Note that \texttt{MortonSort} is used by default, but is typically optional when the filtering posterior is unimodal, as particle ordering has less effect on the soft assignments in this regime (see Appendix~\ref{app:sorting_ablation}). Also note that DSR in Alg.~\ref{alg:dsr} can be implemented using batched matrix operations, hence avoiding explicit loops over particles. 

\begin{algorithm}[H]
\caption{\texttt{MortonSort}: Morton Space-Filling Curve Sort.}
\label{alg:morton_sort}
\begin{algorithmic}[1]
\REQUIRE Particles $\{\mathbf{x}_i\}_{i=1}^N$ with $\mathbf{x}_i \in \mathbb{R}^{d_x}$, resolution $b$ bits
\FOR{$i = 1$ to $N$}
    \FOR{$d = 1$ to $d_x$}
        \STATE $c_i^{(d)} = \left\lfloor \frac{x_i^{(d)} - x_{\min}^{(d)}}{x_{\max}^{(d)} - x_{\min}^{(d)}} \cdot (2^b - 1) \right\rfloor$
        \hfill $\triangleright$ normalize to integer grid
    \ENDFOR
    \STATE Write each $c_i^{(d)}$ in binary: $c_i^{(d)} = \sum_{k=0}^{b-1} c_{i,k}^{(d)} \, 2^k$ with $c_{i,k}^{(d)} \in \{0,1\}$
    \STATE $z_i = \sum_{k=0}^{b-1} \sum_{d=1}^{d_x} c_{i,k}^{(d)} \cdot 2^{k \cdot d_x + (d_x-d)}$
    \hfill $\triangleright$ interleave bits across dimensions
\ENDFOR
\STATE $\pi = \mathrm{argsort}(z_1, \ldots, z_N)$
\RETURN $\{\mathbf{x}_{\pi(i)}\}_{i=1}^N$, $\pi$
\end{algorithmic}
\end{algorithm}

\begin{algorithm}[H]
\caption{Differentiable Systematic Resampling (DSR). Equation marked in \textcolor{green!50!black}{green} is the DSR assignment step, and equation marked in \textcolor{red!50!black}{red} is the SR assignment step.}
\label{alg:dsr}
\begin{algorithmic}[1]
\REQUIRE Weights $\{w_i\}_{i=1}^N$, particles $\{\mathbf{x}_i\}_{i=1}^N$, temperature $\tau > 0$, resolution $b$
\STATE $\{\mathbf{x}_i\}_{i=1}^N, \pi \leftarrow \texttt{MortonSort}(\{\mathbf{x}_i\}_{i=1}^N, b)$; \quad reorder $\{w_i\}$ accordingly
\STATE Compute CDF: $F_0 = 0$, \; $F_i = \sum_{k=1}^i w_k$ for $i = 1, \ldots, N$
\STATE Draw $u_0 \sim \mathrm{Uniform}(0, \tfrac{1}{N})$
\FOR{$j = 1$ to $N$}
    \STATE $u_j = u_0 + \tfrac{j-1}{N}$
    \FOR{$i = 1$ to $N$}
        \STATE \textcolor{red!70!black}{$s_i = \mathds{1}(F_{i-1} < u_j \leq F_i)$}  \hfill $\triangleright$ systematic resampling
        \STATE \textcolor{green!50!black}{$s_i(u_j;\tau) =
        \sigma\!\left(\frac{u_j - F_{i-1}}{\tau}\right)
        -
        \sigma\!\left(\frac{u_j - F_i}{\tau}\right)$}
    \ENDFOR
    \STATE $T_{ij} = s_i(u_j;\tau) \big/ \sum_{k=1}^N s_k(u_j;\tau)$
    \STATE $\tilde{\mathbf{x}}_j = \sum_{i=1}^N T_{ij} \mathbf{x}_i$
\ENDFOR
\RETURN $\{\tilde{\mathbf{x}}_j\}_{j=1}^N$
\end{algorithmic}
\end{algorithm}

\section{Proofs of propositions}\label{sec:proof_of_props}

\begin{proof}[Proof of Proposition 4.1]
We prove both parts of the proposition.

\textbf{Part 1: Convergence of the soft transport plan.}

Fix a column index $j \in [N]$ and define the hard systematic resampling index
\[
a^*_j := \min\{i \in [N] : F_i \ge u_j\}.
\]
By construction, we have
\[
F_{a^*_j-1} < u_j \le F_{a^*_j}.
\]
Since we can view $u_j$ as a realization from a continuous distribution, the event $u_j = F_i$ occurs with probability zero and is ignored in the almost sure sense.

Recall that
\[
s_i(u_j;\tau)
= \sigma\!\left(\frac{u_j - F_{i-1}}{\tau}\right)
- \sigma\!\left(\frac{u_j - F_i}{\tau}\right),
\]
where $\sigma(x) = (1+e^{-x})^{-1}$.

We analyze the limit as $\tau \to 0^+$.

\emph{Case 1: $i = a^*_j$.}
Since $u_j > F_{a^*_j-1}$ and $u_j < F_{a^*_j}$ almost surely, we have
\[
\frac{u_j - F_{a^*_j-1}}{\tau} \to +\infty,
\qquad
\frac{u_j - F_{a^*_j}}{\tau} \to -\infty.
\]
Therefore,
\[
\lim_{\tau\to0^+} s_{a^*_j}(u_j;\tau)
= 1 - 0 = 1
\quad \text{a.s.}
\]

\emph{Case 2: $i < a^*_j$.}
Then $F_i < u_j$ and $F_{i-1} < u_j$, implying both arguments of $\sigma$ tend to $+\infty$ as $\tau\to0^+$.
Hence,
\[
\lim_{\tau\to0^+} s_i(u_j;\tau) = 1 - 1 = 0.
\]

\emph{Case 3: $i > a^*_j$.}
Then $F_{i-1} \ge F_{a^*_j} > u_j$, so both arguments tend to $-\infty$, yielding
\[
\lim_{\tau\to0^+} s_i(u_j;\tau) = 0 - 0 = 0.
\]

Combining all cases, we obtain
\[
\lim_{\tau\to0^+} s_i(u_j;\tau)
=
\begin{cases}
1, & i = a^*_j, \\
0, & i \neq a^*_j,
\end{cases}
\quad \text{a.s.}
\]

Since $\sum_{k=1}^N s_k(u_j;\tau) \to 1$ almost surely, it follows that
\[
\lim_{\tau\to0^+} T^{(\tau)}_{ij}
=
\lim_{\tau\to0^+}
\frac{s_i(u_j;\tau)}{\sum_{k=1}^N s_k(u_j;\tau)}
=
\mathds{1}(i = a^*_j)
\quad \text{a.s.}
\]

\textbf{Part 2: Convergence of the resampled particles.}

Using the result above,
\begin{align*}
\lim_{\tau\to0^+} \tilde{\mathbf{x}}^{(\tau)}_j
&=
\lim_{\tau\to0^+} \sum_{i=1}^N T^{(\tau)}_{ij} \mathbf{x}_i \\
&=
\sum_{i=1}^N \left(\lim_{\tau\to0^+} T^{(\tau)}_{ij}\right) \mathbf{x}_i \\
&=
\mathbf{x}_{a^*_j} \\
&= \tilde{\mathbf{x}}^*_j,
\end{align*}
where the exchange of limit and summation is justified since the sum is finite and $T^{(\tau)}_{ij} \in [0,1]$ for all $\tau>0$.
This completes the proof.
\end{proof}

\begin{proof}[Proof of Proposition 4.2]
Recall that
\[
\tilde{\mathbf{x}}^{(\tau)}_j
=
\sum_{i=1}^N T^{(\tau)}_{ij} \mathbf{x}_i,
\qquad
T^{(\tau)}_{ij}
=
\frac{s_i(u_j;\tau)}{\sum_{k=1}^N s_k(u_j;\tau)}.
\]
Let
\[
\Delta_j := \min\{u_j - F_{a^*_j-1},\, F_{a^*_j} - u_j\} > 0
\quad \text{a.s.}
\]

Since $\sum_{i=1}^N T_{ij}^{(\tau)}=1$, we have $\mathbf{x}_{a^*_j}=\sum_{i=1}^N T_{ij}^{(\tau)}\mathbf{x}_{a^*_j}$. Thus
\begin{align*}
\tilde{\mathbf{x}}^{(\tau)}_j - \mathbf{x}_{a^*_j}
&= \sum_{i=1}^N T_{ij}^{(\tau)}\mathbf{x}_i - \sum_{i=1}^N T_{ij}^{(\tau)}\mathbf{x}_{a^*_j} =
\sum_{i=1}^N T^{(\tau)}_{ij} (\mathbf{x}_i - \mathbf{x}_{a^*_j}).
\end{align*}
Taking norms and using the triangle inequality yields
\[
\|\tilde{\mathbf{x}}^{(\tau)}_j - \mathbf{x}_{a^*_j}\|_2
\le
\sum_{i=1}^N T^{(\tau)}_{ij} \|\mathbf{x}_i - \mathbf{x}_{a^*_j}\|_2
\le
D \sum_{i\neq a^*_j} T^{(\tau)}_{ij},
\]
where $D$ is the diameter, $\sup_{i,j}\|\mathbf{x}_i-\mathbf{x}_j\|_2 \le D$, and we have used the obvious equality $T_{a_j^*j}\|\mathbf{x}_{a_j^*} - \mathbf{x}_{a^*_j}\|_2=0$.

It therefore suffices to bound $\sum_{i\neq a^*_j} T^{(\tau)}_{ij}$.

\textbf{Lower bound on the dominant term.}
Using the inequalities $\sigma(x) \ge 1-e^{-x}$ for $x\ge0$ and $-\sigma(x) \ge -e^{x}$ for $x\le0$, we obtain
\[
s_{a^*_j}(u_j;\tau)
=
\sigma\!\left(\frac{u_j - F_{a^*_j-1}}{\tau}\right)
-
\sigma\!\left(\frac{u_j - F_{a^*_j}}{\tau}\right)
\ge
1 - e^{-\Delta_j/\tau} - e^{-\Delta_j/\tau}.
\]
Thus, for all $\tau \le \Delta_j / \log 4$,
\[
s_{a^*_j}(u_j;\tau) \ge \tfrac12.
\]

\textbf{Upper bound on the remaining terms.}
Fix $i \neq a^*_j$.
If $i < a^*_j$, then $F_i \le F_{a^*_j-1} < u_j$, so both arguments of the sigmoids
are positive and
\[
s_i(u_j;\tau)
=
\sigma\!\left(\frac{u_j - F_{i-1}}{\tau}\right)
-
\sigma\!\left(\frac{u_j - F_i}{\tau}\right)
\le
1 - \sigma\!\left(\frac{u_j - F_{a^*_j-1}}{\tau}\right).
\]
Using $1-\sigma(x)\le e^{-x}$ for $x\ge0$, we obtain
\[
s_i(u_j;\tau)
\le
\exp\!\left(-\frac{u_j - F_{a^*_j-1}}{\tau}\right).
\]

If $i > a^*_j$, then $F_{i-1} \ge F_{a^*_j} > u_j$, so both arguments are negative and
\[
s_i(u_j;\tau)
\le
\sigma\!\left(\frac{u_j - F_{a^*_j}}{\tau}\right)
\le
\exp\!\left(-\frac{F_{a^*_j} - u_j}{\tau}\right),
\]
where we used $\sigma(x)\le e^x$ for $x\le0$.

In both cases, we obtain the uniform bound
\[
s_i(u_j;\tau)
\le
\exp\!\left(-\frac{\Delta_j}{\tau}\right).
\]

Hence,
\[
\sum_{i\neq a^*_j} s_i(u_j;\tau)
\le
N \exp\left(-\frac{\Delta_j}{\tau}\right).
\]

As $s_{a_j^*}(u_j;\tau)\ge \frac{1}{2}$, it follows that $\sum_{k=1}^Ns_k(u_j;\tau)\ge \frac{1}{2}$. Combining the bounds yields
\[
\sum_{i\neq a^*_j} T^{(\tau)}_{ij}
=
\frac{\sum_{i\neq a^*_j} s_i(u_j;\tau)}{\sum_{k=1}^N s_k(u_j;\tau)}
\le \frac{N e^{-\Delta_j/\tau}}{1/2} =
2N \exp\left(-\frac{\Delta_j}{\tau}\right).
\]

Finally,
\[
\|\tilde{\mathbf{x}}^{(\tau)}_j - \mathbf{x}_{a^*_j}\|_2
\le
2DN \exp\!\left(-\frac{\Delta_j}{\tau}\right),
\]
which establishes exponential convergence as $\tau \to 0^+$.
\end{proof}

\section{Architecture and training configuration}\label{sec:training_arch_config}

\subsection{Network Architectures}

All experiments use multilayer perceptrons (MLPs) with ReLU activations to parameterize the transition and proposal models. Both models predict Gaussian distributions with diagonal covariance matrices.

\textbf{Transition Model.} The transition model $f_{\pmb\phi}(\mathbf{x}_t)$ predicts $p_{\pmb\phi}(\mathbf{x}_{t+1}|\mathbf{x}_t)$ and consists of:
\begin{itemize}
    \item Shared trunk: Two fully-connected layers with hidden dimension $h$ and ReLU activations
    \item Mean head: Linear layer mapping from hidden dimension to state dimension $d_x$
    \item Covariance head: Linear layer mapping from hidden dimension to $d_x$, followed by softplus activation to ensure positive diagonal entries
\end{itemize}

The model outputs $(\pmb\mu_{\pmb\phi}(\mathbf{x}_t), \pmb\Lambda_{\pmb\phi}(\mathbf{x}_t))$ where $\pmb\Lambda_{\pmb\phi}$ is a diagonal covariance matrix.

\textbf{Proposal Model.} For simplicity, the proposal model $q_{\pmb\theta}(\mathbf{x}_t|\mathbf{x}_{t-1}, \mathbf{y}_t)$ has the same architecture as the transition model, but takes concatenated inputs $[\mathbf{x}_{t-1}, \mathbf{y}_t]$. 

\subsection{Hyperparameters}

\textbf{Linear Gaussian SSM.} The transition model is parameterized as a learnable matrix $\hat{\mathbf{A}} \in \mathbb{R}^{4 \times 4}$ with a learnable diagonal covariance, and the proposal is a second, separately learned linear-Gaussian model of the same form. The true transition matrix is $\mathbf{A} = 0.9 \cdot \mathrm{blockdiag}(\mathbf{R}(\theta_1), \mathbf{R}(\theta_2))$ with $\theta_1 = 0.3$ and $\theta_2 = 0.5$, where $\mathbf{R}(\theta)$ is the $2 \times 2$ rotation matrix. The process noise covariance is $\mathbf{C}_e = 0.1 \cdot \mathbf{I}_4$ and observations are full-state with $\mathbf{H} = \mathbf{I}_4$. Training uses the Adam optimizer with learning rate $10^{-3}$, batch size $64$, weight decay $0$, and gradient clipping at norm $5.0$. We train with $N = 25$ particles for $500$ epochs on $1000$ trajectories of length $T = 100$ at SMNR $= 10$ dB. For DSR, we set $\tau = 0.1$. We clamp the log-likelihood values below by $-50$ and $-100$ for $10$ dB and $0$ dB SMNR, respectively. This avoids extreme instability during training; none of the methods converges without this. For DSR and EROT, the gradient is clipped elementwise to $[-1, 1]$ per resampling step. EROT uses Sinkhorn with $100$ iterations, with early stopping when the difference is less than $10^{-3}$

\textbf{Lorenz-63.} We use $h = 32$ for both models. The transition model has input dimension 3 (state dimension), while the proposal model has input dimension $3 + 64 = 67$ (state plus observation). Training uses the Adam optimizer with learning rate $5 \times 10^{-4}$, batch size $64$, weight decay $0$, and gradient clipping at norm $5.0$. We train with $N = 25$ particles for $500$ epochs on $1000$ trajectories of length $T = 100$. For DSR, we set $\tau = 0.1$ across all experiments, and use $20$ bits for the Morton curves.

\textbf{Character Typing.} We use a hidden dimension of $h = 32$. The augmented state dimension is $d_x = 9$, with direct observations of the last $3$ dimensions ($d_y = 3$). Training uses $25$ particles, learning rate $1 \times 10^{-3}$, and batch size $64$ on $526$ training sequences of length $T = 178$. All other hyperparameters match Lorenz-63.

\textbf{Evaluation.} We evaluate the state estimation performance using the same training setup except for the resampling method, which uses hard systematic resampling (as this resampling method is typically used in practice). We use $N = 1000$ particles, $T=1000$ (linear and Lorenz-63) or $10$ trajectories at $T=178$ (CharacterTrajectories), and report the values after the last epoch.

\section{SMNR and Evaluation Measures}\label{sec:evaluation_measures}

\subsection{SMNR}\label{app:SMNR}

The signal-to-measurement noise ratio is given by
\begin{equation}
    \text{SMNR} = 10 \log_{10}\!\left( \frac{\sum_{t=1}^{T} \mathbb{E}\!\left[\big\|\mathbf{h}(\mathbf{x}_t^{(j)}) - \mathbb{E}[\mathbf{h}(\mathbf{x}_t^{(j)})]\big\|_2^2\right]}{Td_y \sigma_w^2} \right),
\end{equation}
where $\mathbf{h}(\mathbf{x}_t^{(j)})$ is the noiseless observation, $\sigma_w^2$ is the observation noise variance per dimension, and $d_y$ is the observation dimension. A low SMNR corresponds to high measurement noise.

\subsection{State Estimation}

\textbf{Normalized Mean Squared Error (NMSE).} We evaluate state estimation accuracy using NMSE in dB:
\begin{equation}
\text{NMSE} = 10 \log_{10} \left( \frac{\frac{1}{T}\sum_{t=1}^{T} \|\mathbf{x}_t - \hat{\mathbf{x}}_t\|_2^2}{\frac{1}{T}\sum_{t=1}^{T} \|\mathbf{x}_t\|_2^2} \right),
\end{equation}
where $\mathbf{x}_t$ is the true state, $\hat{\mathbf{x}}_t$ is the estimated state (the posterior mean of the particle filter), and $T$ is the trajectory length. Lower values indicate better performance.

\subsection{Transition Model Quality}

\textbf{One-step Mean Squared Error (MSE).} Measures the quality of the learned dynamics $f_\theta$ by predicting one step ahead from true states:
\begin{equation}
\text{MSE} = \frac{1}{T-1}\sum_{t=1}^{T-1} \|\mathbf{x}_{t+1} - f_\theta(\mathbf{x}_t)\|_2^2.
\end{equation}
Lower values indicate better transition model quality.

\textbf{Multi-step Mean Squared Error (MSE-5-step).} Evaluates prediction over a longer horizon by rolling out the learned dynamics for 5 steps:
\begin{equation}
\text{MSE-5-step} = \frac{1}{T-5}\sum_{t=1}^{T-5} \|\mathbf{x}_{t+5} - f_\theta^{(5)}(\mathbf{x}_t)\|_2^2,
\end{equation}
where $f_\theta^{(5)}(\mathbf{x}_t) = f_\theta(f_\theta(f_\theta(f_\theta(f_\theta(\mathbf{x}_t)))))$ denotes $5$ iterative applications of the learned dynamics. For each use of the transition model, the mean is used as input to the subsequent call of the model.

\textbf{Log-Likelihood (LL).} Evaluates the probabilistic quality of the learned transition model:
\begin{equation}
\text{LL} = \frac{1}{T-1}\sum_{t=1}^{T-1} \log p_\theta(\mathbf{x}_{t+1} | \mathbf{x}_t),
\end{equation}
where $p_\theta(\mathbf{x}_{t+1} | \mathbf{x}_t) = \mathcal{N}(\mathbf{x}_{t+1}; \pmb\mu_\theta(\mathbf{x}_t), \pmb\Lambda_\theta(\mathbf{x}_t))$ is the learned transition distribution, where $f_{\pmb\theta}(\mathbf{x}_t)=(\pmb\mu_\theta(\mathbf{x}_t), \pmb\Lambda_\theta(\mathbf{x}_t))$. Higher values indicate better probabilistic modeling.

\textbf{KL Divergence (KLD).} Measures the distributional mismatch between the learned dynamics and true dynamics:
\begin{equation}
\text{KLD} = \frac{1}{T-1}\sum_{t=1}^{T-1} D_{\mathrm{KL}}\big(p_\phi(\mathbf{x}_{t+1} \mid \mathbf{x}_t) \,\Vert\, p_{\mathrm{true}}(\mathbf{x}_{t+1} \mid \mathbf{x}_t)\big).
\end{equation}
Lower values indicate a closer match to the true transition distribution.

\section{Empirical Analysis of the DSR Relaxation}\label{app:empirical_analysis}

The theoretical results in Section~\ref{sec:theory} establish that DSR converges to systematic resampling as $\tau \to 0^+$, with pointwise exponential rates. In this section, we complement the theory with an empirical study of the DSR transport plan and the resampled particles. We use a sequential importance resampling (SIR) particle filter with systematic resampling~\cite{tutorial_particle_filter_2011} on the Lorenz-63 system with oracle (true) dynamics, at $10$ dB SMNR with the Gaussian point spread observation map. This isolates the resampling relaxation from the dynamics learning, so we can study DSR as an approximation to systematic resampling directly. We vary $\tau \in \{0.01, 0.05, 0.1, 0.5, 1.0\}$ and $N \in \{25, 50, 100, 200\}$, and report results averaged over 10 sequences of length $T = 100$.

\subsection{Transport Plan Visualization}

Figure~\ref{fig:transport_maps_tau} shows the DSR transportation matrices $T^{(\tau)}$ for varying $\tau$, alongside hard systematic resampling and EROT ($\varepsilon = 0.5$), at a representative time step with $N = 25$ particles. At $\tau = 0.01$, the DSR transport plan closely resembles the sparse, block-diagonal structure of hard systematic resampling, with mass concentrated on a single ancestor per resampled particle. As $\tau$ grows, the plan becomes more diffuse and spreads mass across neighbouring ancestors. At $\tau = 0.5$, the plan becomes visibly banded: a high-weight ancestor contributes to many resampled particles, so DSR becomes less restrictive about spatial locality. This brings DSR qualitatively closer to EROT. A main difference is that at $\tau = 0.5$ DSR's columns are flatter: each resampled particle mixes more ancestors than under EROT ($N_{eff} = 11.2$ vs. $9.8$ for $N = 25$, Appendix \ref{app:effective_ancestors}), which makes DSR less local. A lower column-wise spread implies that the resampling linear combination has more even weights, which in turn becomes less restrictive to spatial locality. On the other hand, EROT's transport plan bears little resemblance to systematic resampling overall, whereas DSR retains the CDF-ordered banded structure even at high $\tau$.

\begin{figure}[t]
\centering
\includegraphics[width=\textwidth]{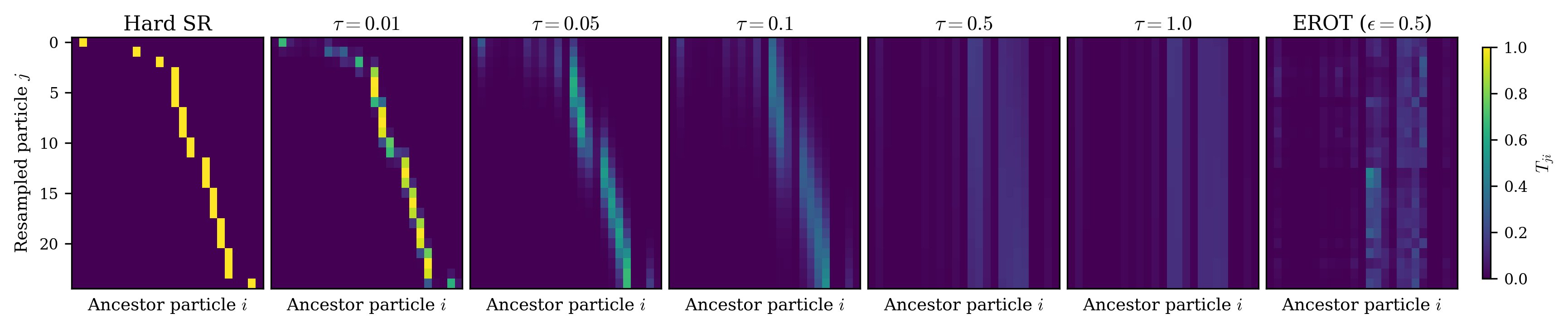}
\caption{Transportation maps for hard systematic resampling and DSR at varying temperatures $\tau \in \{0.01, 0.05, 0.1, 0.5, 1.0\}$, alongside EROT ($\varepsilon = 0.5$), for a SIR particle filter of the Lorenz-63 system with $N = 25$ particles. DSR is shown without Morton sorting (the filter's particle order).}
\label{fig:transport_maps_tau}
\end{figure}

\subsection{Effective Number of Ancestors}\label{app:effective_ancestors}

To quantify the sharpness of the transport plan, we compute the effective number of ancestors for each resampled particle $j$:
\[
N_{\mathrm{eff},j} = \exp\!\left(-\sum_{i=1}^N T_{ij} \log T_{ij}\right),
\]
averaged over particles and time steps. For hard systematic resampling, $N_{\mathrm{eff},j} = 1$, as each resampled particle copies exactly one ancestor.

Figure~\ref{fig:eff_ancestors} shows $N_{\mathrm{eff}}$ as a function of $\tau$ for different particle counts. At low temperatures, the DSR transport plan is sharply concentrated: $N_{\mathrm{eff}} = 1.70$ at $\tau = 0.01$ with $N = 25$, close to the hard SR limit. $N_{eff}$ grows with $\tau$ up to $\tau = 0.5$ and then levels off ($N = 25$: $11.2$ at $\tau = 0.5$, $10.6$ at $\tau = 1.0$); for fixed $\tau$ it scales with $N$. This is consistent with the $N$-dependent factor in Proposition~\ref{prop:exponential_convergence}: as $N$ grows, the CDF intervals $[F_{i-1}, F_i]$ become narrower, so more ancestors fall within the sigmoid bandwidth $\tau$ and contribute to each resampled particle. 

For comparison, EROT ($\varepsilon = 0.5$) yields $N_{\mathrm{eff}}$ of $9.83$, $18.36$, $36.41$, and $71.77$ for $N = 25, 50, 100, 200$. For DSR using the default value $\tau = 0.1$, DSR is substantially sparser than EROT across all particle counts (e.g., $6.54$ vs.\ $9.83$ for $N = 25$; $23.11$ vs.\ $36.41$ for $N = 100$).




\subsection{L2 Distance and Effect of Morton Sorting}

We measure the per-particle L2 distance between DSR and hard SR resampled particles:
\[
\mathrm{L2}_{\mathrm{particles}} = \frac{1}{N}\sum_{j=1}^N \|\tilde{\mathbf{x}}^{(\tau)}_j - \mathbf{x}_{a_j^*}\|_2,
\]
where $\mathbf{x}_{a_j^*}$ is the resampled particle by systematic resampling. This directly relates to the bound in Proposition~\ref{prop:exponential_convergence}. Further, to assess the role of particle ordering, we compare Morton space-filling curve sorting (used in all VSMC experiments) against the unsorted order the filter produces (inherited from the previous resampling step).

Figure~\ref{fig:l2_distance} shows the L2 distance as a function of $\tau$ for both orderings. As suggested in Proposition~\ref{prop:exponential_convergence}, the distance grows with $\tau$ under Morton sorting: for $N = 25$, it increases from $0.10$ at $\tau = 0.01$ to $0.51$ at $\tau = 0.5$. Morton sorting consistently reduces the L2 distance relative to the unsorted order, and the gap is largest at low $\tau$ and high $N$. At $\tau = 0.01$ with $N = 200$, Morton sorting reduces the distance from $0.38$ to $0.25$, a $35\%$ reduction. The reason is that at low $\tau$, each resampled particle draws mass primarily from its CDF neighbours, and Morton sorting ensures those neighbours are also spatially close, keeping the mixture near the hard ancestor. At higher temperatures ($\tau \geq 0.5$), the plan is diffuse enough that ordering matters less.

\begin{figure}[t]
\centering
\begin{subfigure}[t]{0.48\textwidth}
    \centering
    \includegraphics[width=\textwidth]{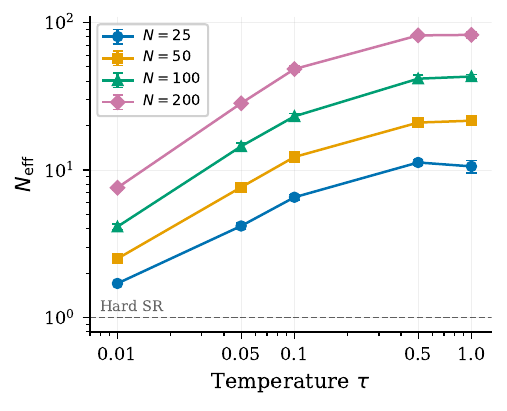}
    \caption{Effective number of ancestors $N_{\mathrm{eff}}$ as a function of temperature $\tau$ for varying particle counts $N$. The dashed line indicates hard systematic resampling ($N_{\mathrm{eff}} = 1$).}
    \label{fig:eff_ancestors}
\end{subfigure}
\hfill
\begin{subfigure}[t]{0.48\textwidth}
    \centering
    \includegraphics[width=\textwidth]{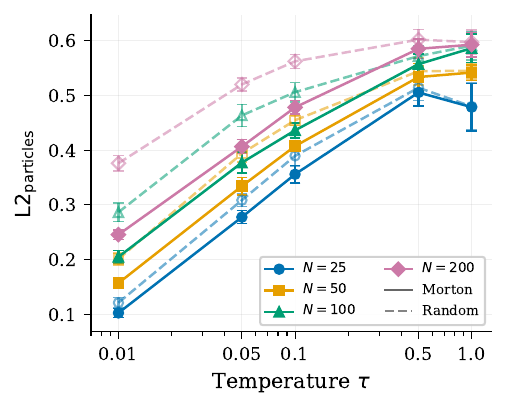}
    \caption{Per-particle L2 distance between DSR and hard SR resampled particles. Solid lines use Morton sorting; dashed lines use the filter's unsorted particle order.}
    \label{fig:l2_distance}
\end{subfigure}
\caption{Empirical analysis of the DSR relaxation on the Lorenz-63 system with oracle dynamics at SMNR $= 10$~dB. Error bars show $\pm 1$ standard deviation over 10 sequences.}
\label{fig:dsr_empirical}
\end{figure}

\subsection{Bias and Variance of DSR Estimates}\label{app:dsr_bias_var} Here we study the bias and variance of DSR. For resampling schemes, a typical study is the study of how $I = \sum_i w_i f(\mathbf{x}_i)$ ($f$ is an arbitrary function) changes after resampling, in particular, studying how $I$ behaves after resampling. To quantify the effect, we measure the bias, variance, and MSE of a test-function $f$ estimate under each scheme. Fixing a particle set and weights, over $20000$ seeds we compute $\hat{I} = \sum_j \tilde{w}_j f(\tilde{\mathbf{x}}_j)$ with equal weights $\tilde{w}_j = 1/N$ after resampling, and compare against the weighted estimate $I = \sum_i w_i f(\mathbf{x}_i)$. We use $N = 100$ particles with $\mathbf{x}_j \sim \mathcal{N}(\mathbf{0}, \mathbf{I}_3)$ and $w_j$ the likelihood, and two test functions, $f(\mathbf{x}) = \mathbf{x}$ and $f(\mathbf{x}) = \mathbf{x}^2$ (elementwise). Hard systematic and multinomial resampling are unbiased, so their reported bias is Monte-Carlo error. Results are in Table~\ref{tab:dsr_bias_var}.

\begin{table}[ht]
\centering
\caption{Bias, variance, and MSE of resampling estimates against the weighted estimate $I$, over $20000$ seeds with $N = 100$. Hard SR and multinomial are unbiased (their bias is Monte-Carlo error).}
\label{tab:dsr_bias_var}
\setlength{\tabcolsep}{4pt}
\small
\begin{tabular}{lcccccc}
\toprule
& \multicolumn{3}{c}{$f(\mathbf{x}) = \mathbf{x}$} & \multicolumn{3}{c}{$f(\mathbf{x}) = \mathbf{x}^2$} \\
\cmidrule(lr){2-4} \cmidrule(lr){5-7}
Method & bias & var & MSE & bias & var & MSE \\
\midrule
Hard SR & $2.3 \times 10^{-4}$ & $1.2 \times 10^{-3}$ & $1.2 \times 10^{-3}$ & $1.2 \times 10^{-4}$ & $7.0 \times 10^{-4}$ & $7.0 \times 10^{-4}$ \\
Multinomial & $4.9 \times 10^{-4}$ & $5.5 \times 10^{-3}$ & $5.5 \times 10^{-3}$ & $2.8 \times 10^{-4}$ & $2.6 \times 10^{-3}$ & $2.6 \times 10^{-3}$ \\
DSR ($\tau = 0.01$) & $2.4 \times 10^{-3}$ & $2.9 \times 10^{-5}$ & $3.4 \times 10^{-5}$ & $1.2 \times 10^{-1}$ & $3.8 \times 10^{-7}$ & $1.4 \times 10^{-2}$ \\
DSR ($\tau = 0.05$) & $1.2 \times 10^{-2}$ & $8.6 \times 10^{-6}$ & $1.6 \times 10^{-4}$ & $2.3 \times 10^{-1}$ & $3.1 \times 10^{-7}$ & $5.3 \times 10^{-2}$ \\
DSR ($\tau = 0.1$) & $2.0 \times 10^{-2}$ & $4.7 \times 10^{-6}$ & $3.9 \times 10^{-4}$ & $2.7 \times 10^{-1}$ & $2.0 \times 10^{-7}$ & $7.3 \times 10^{-2}$ \\
DSR ($\tau = 0.5$) & $1.2 \times 10^{-2}$ & $1.6 \times 10^{-7}$ & $1.4 \times 10^{-4}$ & $3.1 \times 10^{-1}$ & $2.2 \times 10^{-9}$ & $9.9 \times 10^{-2}$ \\
DSR ($\tau = 1.0$) & $3.6 \times 10^{-3}$ & $1.3 \times 10^{-8}$ & $1.3 \times 10^{-5}$ & $3.2 \times 10^{-1}$ & $2.0 \times 10^{-10}$ & $1.0 \times 10^{-1}$ \\
\bottomrule
\end{tabular}
\end{table}

For the linear function $f(\mathbf{x}) = \mathbf{x}$, DSR trades variance for bias at every $\tau$: the variance is $1.6$ to $5$ orders of magnitude below hard systematic resampling, since a convex combination is an average, and DSR attains lower MSE at every $\tau$. Here the bias is non-monotone in $\tau$, because both limits are favourable for a linear function: $\tau \to 0^+$ recovers systematic resampling, while large $\tau$ drives every output toward the weighted mean, which estimates a linear function exactly (here $I \approx 0$, so the contraction at $\tau = 1.0$ is favourable and yields the lowest MSE). This should not be read as $\tau = 1.0$ being preferable in general, since the same contraction degrades nonlinear functions: for $f(\mathbf{x}) = \mathbf{x}^2$ the bias increases monotonically with $\tau$. The variance reduction comes from averaging: every output is a convex combination of ancestors.

\subsection{Discussion}

The empirical analysis supports the theoretical predictions of Section~\ref{sec:theory}. The DSR transport plan converges to the sparse structure of systematic resampling as $\tau \to 0^+$, with $N_{\mathrm{eff}} \to 1$ at low temperatures. At the default value $\tau = 0.1$, DSR is substantially sparser than EROT ($N_{eff}$ $6.5$ vs. $9.8$ for $N = 25$) and puts on average $35\%$ of each resampled particle's weight on its hard SR ancestor ($N = 25$). Morton sorting reduces the approximation error, particularly at low $\tau$ and high $N$, confirming that locality-preserving orderings complement the DSR relaxation. Together, these results show that DSR approximates systematic resampling closely at small $\tau$, and how the approximation loosens as $\tau$ and $N$ grow.

\section{Ablation Studies}
\label{app:ablation}

\subsection{Dependence on $\tau$}
\label{app:tau_ablation}

We study the effect of the temperature parameter $\tau$ on filtering performance 
in the Lorenz-63 dataset at SMNR = 10~dB. Table~\ref{tab:ablation} reports NMSE 
results for different temperature values.

\begin{table}[ht]
\centering
\caption{Ablation study of the temperature parameter $\tau$ on Lorenz-63 
(SMNR = 10~dB).}
\label{tab:ablation}
\begin{tabular}{cc}
\toprule
Temperature $\tau$ & NMSE (dB) \\
\midrule
0.01 & $-28.31 \pm 0.15$ \\
0.05 & $-29.56 \pm 0.34$ \\
0.10 & $\underline{-29.94} \pm 0.28$ \\
0.50 & $\mathbf{-30.21 \pm 0.29}$ \\
1.00 & $-29.91 \pm 0.28$ \\
\bottomrule
\end{tabular}
\end{table}

At $\tau = 0.01$, NMSE is notably worse ($-28.31$ dB) than at larger temperatures. This is consistent with the $\mathcal{O}(1/\tau)$ scaling of the sigmoid derivative,
\[
\frac{d}{dx}\sigma(x) = \frac{1}{\tau}\sigma(x/\tau)\bigl(1-\sigma(x/\tau)\bigr),
\]
which causes gradient magnitudes to explode during backpropagation at small $\tau$, leading to unstable optimization. As $\tau$ grows, gradient magnitudes become more moderate and training stabilizes. Performance plateaus for $\tau \in [0.05, 1.0]$, with NMSE varying by less than $0.7$ dB across this range and consistently low run-to-run variance. Based on these results, we fix $\tau = 0.1$ for all experiments; this is a conservative choice within the stable range that keeps the soft transport plan close to hard systematic resampling (see Appendix~\ref{app:empirical_analysis}).

Beyond NMSE, the temperature $\tau$ governs a bias--variance tradeoff in the gradient itself. Table~\ref{tab:tau_gradient} reports, for the same setup, the gradient-norm statistics and the per-particle $L_2$ distance of the DSR transport to hard systematic resampling (Appendix~\ref{app:empirical_analysis}) as a function of $\tau$. The two quantities move in opposite directions: as $\tau$ shrinks, the approximation error $L_2$ decreases ($0.505 \to 0.102$) while the median gradient norm and within-epoch standard deviation increase. The ratio of median within-epoch std to median gradient norm stays close to $0.45$ for $\tau \ge 0.05$, so gradients grow but remain comparably well behaved; at $\tau = 0.01$ the ratio rises to $0.55$, indicating larger variance relative to norm and less stable training, consistent with the degraded NMSE at $\tau = 0.01$. We read $\tau = 0.01$ as a change of regime rather than of scale.

\begin{table}[ht]
\centering
\caption{Gradient behaviour and approximation error as a function of $\tau$ on Lorenz-63 (SMNR $= 10$~dB). The $L_2$ column is the per-particle distance of the DSR transport to hard systematic resampling from the same particle set and offset (Appendix~\ref{app:empirical_analysis}).}
\label{tab:tau_gradient}
\setlength{\tabcolsep}{4pt}
\small
\begin{tabular}{lcccccc}
\toprule
& & \multicolumn{2}{c}{Gradient norm} & \multicolumn{2}{c}{Within-epoch std} & \\
\cmidrule(lr){3-4} \cmidrule(lr){5-6}
$\tau$ & NMSE (dB) & Median & Mean & Median & Mean & $L_2$ to hard SR \\
\midrule
$0.01$ & $-28.31{\scriptstyle\pm 0.15}$ & $14842$ & $2.39 \times 10^{5}$ & $8163$ & $4.35 \times 10^{5}$ & $0.102{\scriptstyle\pm 0.009}$ \\
$0.05$ & $-29.56{\scriptstyle\pm 0.34}$ & $6457$ & $7228$ & $3009$ & $3676$ & $0.277{\scriptstyle\pm 0.012}$ \\
$0.10$ & $-29.94{\scriptstyle\pm 0.28}$ & $4174$ & $4504$ & $1885$ & $2720$ & $0.356{\scriptstyle\pm 0.017}$ \\
$0.50$ & $\mathbf{-30.21}{\scriptstyle\pm 0.29}$ & $2546$ & $2951$ & $1155$ & $1985$ & $0.505{\scriptstyle\pm 0.026}$ \\
$1.00$ & $-29.91{\scriptstyle\pm 0.28}$ & $2306$ & $3041$ & $1027$ & $3453$ & $0.479{\scriptstyle\pm 0.046}$ \\
\bottomrule
\end{tabular}
\end{table}

\subsection{Annealing $\tau$ during training}\label{app:annealing}
The temperature is fixed at $\tau = 0.1$ in all main experiments. Here we test annealing $\tau$ over the course of training: starting at a larger $\tau$ for smoother gradients and decreasing it toward a sharper, more faithful assignment. We anneal $\tau$ linearly from $1.0$ to $0.1$ over $300$ epochs (training is done for $500$ epochs), and evaluate on Lorenz-63 at SMNR $= 10$~dB under the camera measurement map, with all other parameters at default. Table~\ref{tab:annealing} compares the annealed schedule against the fixed $\tau = 0.1$ setting.

\begin{table}[ht]
\centering
\caption{Fixed versus annealed $\tau$ on Lorenz-63 (SMNR $= 10$~dB). Results show mean $\pm$ std over 10 runs.}
\label{tab:annealing}
\small
\begin{tabular}{lccccc}
\toprule
& State Est. & \multicolumn{4}{c}{Transition Model Quality} \\
\cmidrule(lr){2-2} \cmidrule(lr){3-6}
Setting & NMSE (dB) & MSE & MSE-5-step & LL & KLD \\
\midrule
Fixed ($\tau = 0.1$) & $-29.94{\scriptstyle\pm 0.28}$ & $0.47{\scriptstyle\pm 0.24}$ & $3.38{\scriptstyle\pm 2.11}$ & $-1.98{\scriptstyle\pm 0.27}$ & $3.60{\scriptstyle\pm 1.04}$ \\
Annealed & $-30.33{\scriptstyle\pm 0.10}$ & $0.40{\scriptstyle\pm 0.11}$ & $3.01{\scriptstyle\pm 0.63}$ & $-1.86{\scriptstyle\pm 0.13}$ & $3.15{\scriptstyle\pm 0.32}$ \\
\bottomrule
\end{tabular}
\end{table}

Annealing gives a small but consistent improvement over the fixed setting across all metrics, and notably reduces run-to-run variance (e.g.\ NMSE std $0.10$ versus $0.28$ for fixed $\tau$). This indicates that annealing can benefit learning; we leave a systematic study of annealing schedules for future work.

\subsection{Dependence on $N$}\label{app:N_ablation}

We study the effect of the number of particles $N$ on filtering and dynamics learning performance for all resampling methods on the Lorenz-63 system at SMNR $= 10$~dB using the Gaussian point spread observation map. All other hyperparameters are kept fixed. Results are shown in Table~\ref{tab:n_ablation}.

\begin{table}
\centering
\caption{Effect of particle count $N$ on resampling methods for Lorenz-63 at SMNR $= 10$~dB. Results show mean $\pm$ std over 10 runs. OOM indicates that training exceeded available GPU memory.}
\label{tab:n_ablation}
\begin{tabular}{clccccc}
\toprule
& & State Est. & \multicolumn{4}{c}{Transition Model Quality} \\
\cmidrule(lr){3-3} \cmidrule(lr){4-7}
$N$ & Method & NMSE (dB) & MSE & MSE-5-step & LL & KLD \\
\midrule
25  & Hard & $\underline{-28.86}{\scriptstyle\pm 0.41}$ & $\underline{0.55}{\scriptstyle\pm 0.19}$ & $5.50{\scriptstyle\pm 3.68}$ & $\underline{-2.26}{\scriptstyle\pm 0.25}$ & $\underline{5.41}{\scriptstyle\pm 1.47}$ \\
    & Soft & $-27.82{\scriptstyle\pm 0.45}$ & $0.86{\scriptstyle\pm 0.49}$ & $6.45{\scriptstyle\pm 3.92}$ & $-2.55{\scriptstyle\pm 0.48}$ & $6.90{\scriptstyle\pm 2.43}$ \\
    & EROT & $-28.76{\scriptstyle\pm 0.31}$ & $0.56{\scriptstyle\pm 0.24}$ & $\underline{3.39}{\scriptstyle\pm 2.59}$ & $-2.32{\scriptstyle\pm 0.20}$ & $5.74{\scriptstyle\pm 1.03}$ \\
    & DSR  & $\mathbf{-29.94}{\scriptstyle\pm 0.28}$ & $\mathbf{0.47}{\scriptstyle\pm 0.24}$ & $\mathbf{3.38}{\scriptstyle\pm 2.11}$ & $\mathbf{-1.98}{\scriptstyle\pm 0.27}$ & $\mathbf{3.60}{\scriptstyle\pm 1.04}$ \\
\midrule
50  & Hard & $-28.30{\scriptstyle\pm 0.88}$ & $\mathbf{0.39}{\scriptstyle\pm 0.08}$ & $\mathbf{1.73}{\scriptstyle\pm 0.61}$ & $\underline{-1.98}{\scriptstyle\pm 0.09}$ & $\underline{4.30}{\scriptstyle\pm 0.66}$ \\
    & Soft & $-28.30{\scriptstyle\pm 0.97}$ & $0.59{\scriptstyle\pm 0.35}$ & $5.78{\scriptstyle\pm 3.19}$ & $-2.33{\scriptstyle\pm 0.45}$ & $6.27{\scriptstyle\pm 2.71}$ \\
    & EROT & $-27.91{\scriptstyle\pm 0.39}$ & $0.59{\scriptstyle\pm 0.21}$ & $5.22{\scriptstyle\pm 2.95}$ & $-2.49{\scriptstyle\pm 0.21}$ & $6.88{\scriptstyle\pm 1.44}$ \\
    & DSR  & $\mathbf{-30.34}{\scriptstyle\pm 0.28}$ & $\underline{0.41}{\scriptstyle\pm 0.13}$ & $\underline{1.95}{\scriptstyle\pm 1.04}$ & $\mathbf{-1.77}{\scriptstyle\pm 0.15}$ & $\mathbf{2.86}{\scriptstyle\pm 0.67}$ \\
\midrule
100 & Hard & $-28.34{\scriptstyle\pm 0.48}$ & $0.79{\scriptstyle\pm 0.21}$ & $4.65{\scriptstyle\pm 1.52}$ & $-2.61{\scriptstyle\pm 0.30}$ & $7.20{\scriptstyle\pm 3.42}$ \\
    & Soft & $-28.93{\scriptstyle\pm 0.30}$ & $0.50{\scriptstyle\pm 0.24}$ & $3.35{\scriptstyle\pm 1.74}$ & $-2.00{\scriptstyle\pm 0.30}$ & $3.61{\scriptstyle\pm 0.75}$ \\
    & EROT & $-27.55{\scriptstyle\pm 0.92}$ & $\mathbf{0.39}{\scriptstyle\pm 0.19}$ & $\mathbf{1.86}{\scriptstyle\pm 1.20}$ & $-2.03{\scriptstyle\pm 0.25}$ & $4.29{\scriptstyle\pm 0.96}$ \\
    & DSR  & $\mathbf{-30.42}{\scriptstyle\pm 0.20}$ & $\underline{0.44}{\scriptstyle\pm 0.19}$ & $\underline{2.38}{\scriptstyle\pm 0.83}$ & $\mathbf{-1.80}{\scriptstyle\pm 0.28}$ & $\mathbf{2.82}{\scriptstyle\pm 0.52}$ \\
\midrule
200 & Hard & $\underline{-20.42}{\scriptstyle\pm 7.29}$ & $0.75{\scriptstyle\pm 0.32}$ & $8.92{\scriptstyle\pm 4.96}$ & $\underline{-2.54}{\scriptstyle\pm 0.30}$ & $\underline{8.27}{\scriptstyle\pm 2.74}$ \\
    & Soft & $-21.52{\scriptstyle\pm 12.20}$ & $\underline{0.66}{\scriptstyle\pm 0.43}$ & $\underline{5.39}{\scriptstyle\pm 4.65}$ & $\underline{-2.54}{\scriptstyle\pm 0.66}$ & $9.62{\scriptstyle\pm 9.44}$ \\
    & EROT & OOM & OOM & OOM & OOM & OOM \\
    & DSR  & $\mathbf{-30.35}{\scriptstyle\pm 0.22}$ & $\mathbf{0.37}{\scriptstyle\pm 0.17}$ & $\mathbf{3.27}{\scriptstyle\pm 2.23}$ & $\mathbf{-1.73}{\scriptstyle\pm 0.23}$ & $\mathbf{2.52}{\scriptstyle\pm 0.41}$ \\
\bottomrule
\end{tabular}
\end{table}

DSR is the only method that scales gracefully across all tested particle counts. Its NMSE improves from $-29.94$~dB at $N = 25$ to around $-30.4$~dB from $N = 50$ onwards and plateaus there, with consistently low variance ($\leq 0.28$~dB std) across runs. The transition model metrics show the same trend: MSE, LL, and KLD all improve with $N$ and reach their best values at $N = 200$ (except MSE-5-step, which generally has a high variance).

By comparison, the other methods either fail to train stably or run out of memory. Hard and Soft are competitive at $N \leq 100$, but at $N = 200$ their performance collapses, with NMSE degrading to $-20.42$ and $-21.52$~dB and standard deviations above $7$ and $12$~dB across runs. These large variances are the result of training instability; there is a large variance across runs; for example, one run for Soft achieved an NMSE of $0.23$ dB. A likely cause is that, at large $N$, the variance of the biased estimator $\nabla \mathcal{L} \approx g_{\mathrm{rep}}$ (used by Hard) grows because more ancestor decisions are dropped from the gradient graph, and Soft's importance-weight correction becomes numerically less stable as the uniform component of its sampling distribution $(1-\alpha)/N$ shrinks. EROT does not complete at $N = 200$ because the $\mathcal{O}(N^2)$ transport matrix, combined with $100$ Sinkhorn iterations per resampling step and the autodiff tape over $T = 100$ time steps, exceeds the available GPU memory.

These results illustrate two distinct scaling properties of DSR. First, it provides a stable gradient signal across particle counts, whereas Hard and Soft degrade. Second, it needs about a quarter of EROT's memory per resampling step (no Sinkhorn iterations, float32), so it still fits at $N = 200$, whereas EROT does not. The relaxation parameter ($\tau$) of DSR is kept fixed at $0.1$ across all $N$.

\subsection{Dependence on spatial sorting of CDF}\label{app:sorting_ablation}

We now study the effect of Morton sorting on VSMC performance. As discussed, DSR operates on the empirical CDF, and the particle ordering determines which ancestors interact through the soft transport plan. Morton space-filling curve sorting places CDF neighbors close in state space, producing a transport plan that more closely approximates hard systematic resampling (Appendix~\ref{app:empirical_analysis}). The alternative is the unsorted order the filter produces (inherited from the previous resampling step), which allows soft assignments to mix spatially distant particles. We compare the two orderings with DSR on Lorenz-63 under the Gaussian point spread observation map~\eqref{eq:lorenz_obs_map}.

\begin{table}[ht]
\centering
\caption{Effect of particle ordering on DSR performance for Lorenz-63 at
SMNR $= 10$ and $0$~dB. Results show mean $\pm$ std over 10 runs.}
\label{tab:sorting_ablation}
\small
\begin{tabular}{lccccc}
\toprule
 & State Est. & \multicolumn{4}{c}{Transition Model Quality} \\
\cmidrule(lr){2-2} \cmidrule(lr){3-6}
Ordering & NMSE (dB) & MSE & MSE-5-step & LL & KLD \\
\midrule
\multicolumn{6}{c}{SMNR $= 10$~dB} \\
\midrule
Morton & $-29.94{\scriptstyle\pm 0.28}$ & $0.47{\scriptstyle\pm 0.24}$ & $3.38{\scriptstyle\pm 2.11}$ & $-1.98{\scriptstyle\pm 0.27}$ & $3.60{\scriptstyle\pm 1.04}$ \\
Unsorted & $-29.96{\scriptstyle\pm 0.34}$ & $0.44{\scriptstyle\pm 0.20}$ & $2.31{\scriptstyle\pm 0.80}$ & $-1.92{\scriptstyle\pm 0.21}$ & $3.59{\scriptstyle\pm 0.80}$ \\
\midrule
\multicolumn{6}{c}{SMNR $= 0$~dB} \\
\midrule
Morton & $-18.28{\scriptstyle\pm 0.46}$ & $0.54{\scriptstyle\pm 0.33}$ & $3.70{\scriptstyle\pm 2.87}$ & $-3.06{\scriptstyle\pm 0.24}$ & $13.41{\scriptstyle\pm 1.94}$ \\
Unsorted & $-18.43{\scriptstyle\pm 0.59}$ & $0.68{\scriptstyle\pm 0.18}$ & $6.55{\scriptstyle\pm 4.64}$ & $-3.19{\scriptstyle\pm 0.15}$ & $14.43{\scriptstyle\pm 2.23}$ \\
\bottomrule
\end{tabular}
\end{table}

The results can be seen in Table~\ref{tab:sorting_ablation}. At $10$~dB, the two orderings are essentially indistinguishable, with overlapping error bars on every metric. At $0$~dB, the picture is mixed: Morton sorting gives a better transition model (MSE-5-step of $3.70$ vs.\ $6.55$, but within one standard deviation), while unsorted ordering gives marginally better state estimation ($-18.43$ vs.\ $-18.28$~dB NMSE), though this is within one standard deviation. A plausible reading is that the two orderings trade off along the dual role of the soft transport plan. Morton sorting produces a closer approximation to hard systematic resampling, preserving the spatial coherence that helps dynamics learning when weights become concentrated. Unsorted ordering mixes more distant particles, which softens the relaxation further; this extra diversity may help state estimation through broader particle coverage, but at the cost of a less structured gradient signal for the transition model. Overall, DSR is relatively robust to the choice of ordering, with Morton sorting offering a modest advantage for transition model learning at low SNR. Studying the interaction between particle ordering and posterior geometry, for instance, under multimodal or heavy-tailed posteriors, is an interesting direction for future work.

\subsection{EROT regularization sensitivity at low SNR}
\label{app:erot_eps}

Table~\ref{tab:erot_eps_ablation} reports EROT performance on Lorenz-63 at $0$ dB SMNR for $\varepsilon \in \{0.5, 1.0\}$, alongside DSR for reference. The default $\varepsilon = 0.5$, as used in~\cite{corenflos2021differentiable, vannierop2025deepvsmc}, exhibits severe training instability at this noise level, with MSE-5-step standard deviation exceeding the mean. Increasing $\varepsilon$ to $1.0$ stabilizes training and improves mean performance across all metrics, but DSR still outperforms EROT on every measure. Training time drops slightly with $\varepsilon = 1.0$ ($40.35$ s vs. $42.94$ s per epoch), consistent with the Sinkhorn iterations stopping earlier at the larger $\varepsilon$.

\begin{table}[ht]
\centering
\caption{EROT regularization sensitivity on Lorenz-63 at SMNR = $0$ dB.
Results show mean $\pm$ std over $10$ runs. DSR with $\tau = 0.1$ is
included for reference. Best in \textbf{bold}.}
\label{tab:erot_eps_ablation}
\setlength{\tabcolsep}{3pt}
\small
\begin{tabular}{lcccccc}
\toprule
& \multicolumn{1}{c}{State Est.} & \multicolumn{4}{c}{Transition Model Quality} & \\
\cmidrule(lr){2-2} \cmidrule(lr){3-6}
Method & NMSE (dB) & MSE & MSE-5-step & LL & KLD & Time/epoch (s) \\
\midrule
EROT ($\varepsilon = 0.5$) & $-12.57{\scriptstyle\pm 1.64}$ & $4.79{\scriptstyle\pm 5.65}$ & $48.41{\scriptstyle\pm 60.72}$ & $-4.28{\scriptstyle\pm 1.54}$ & $33.82{\scriptstyle\pm 29.79}$ & $42.94{\scriptstyle\pm 0.77}$ \\
EROT ($\varepsilon = 1.0$) & $-13.47{\scriptstyle\pm 0.95}$ & $0.57{\scriptstyle\pm 0.43}$ & $9.19{\scriptstyle\pm 7.58}$ & $-3.49{\scriptstyle\pm 0.09}$ & $20.20{\scriptstyle\pm 2.33}$ & $40.35{\scriptstyle\pm 0.39}$ \\
DSR ($\tau = 0.1$) & $\mathbf{-18.28}{\scriptstyle\pm 0.46}$ & $\mathbf{0.54}{\scriptstyle\pm 0.33}$ & $\mathbf{3.70}{\scriptstyle\pm 2.87}$ & $\mathbf{-3.06}{\scriptstyle\pm 0.24}$ & $\mathbf{13.41}{\scriptstyle\pm 1.94}$ & $\mathbf{18.39}{\scriptstyle\pm 0.16}$ \\
\bottomrule
\end{tabular}
\end{table}

\section{Multimodal Posteriors and Inter-modal Mass}\label{app:multimodal}
At finite $\tau$, DSR forms convex combinations of ancestors, which can place mass between well-separated modes of a multimodal posterior. To study this phenomenon, we study a static two-component Gaussian (i.e., a mixture of two Gaussians). We draw particles from this two-component Gaussian, perform resampling, and study the intermodal mass. The two-component Gaussian is given by
\begin{equation}
\tfrac{1}{2}\,\mathcal{N}\!\big(-\tfrac{\Delta}{2}\mathbf{e}_0,\, \sigma^2 \mathbf{I}_d\big) + \tfrac{1}{2}\,\mathcal{N}\!\big(+\tfrac{\Delta}{2}\mathbf{e}_0,\, \sigma^2 \mathbf{I}_d\big),
\end{equation}
where $\Delta$ is the mode separation, $\mathbf{e}_0 = (1, 0, \ldots, 0)^\top$, $d$ is the dimension, and $\sigma^2$ is a scalar for the covariance. The modes are separated along coordinate $0$ only, while coordinates $1, \ldots, d-1$ carry isotropic noise and no mode information. We use $N = 100$, $\Delta = 10$, $\sigma = 0.3$ (a separation of $\approx 33$ within-mode standard deviations), and measure the resampled mass in the inter-modal region $|x_0| < 3$. Particles are randomly permuted before resampling, and we use the manuscript defaults ($\alpha = 0.5$; $\tau = 0.1$; $\varepsilon = 0.5$ with at most $100$ Sinkhorn iterations). We report the inter-modal mass as mean $\pm$ std over $10^4$ simulations.

Any scheme that copies ancestors places zero mass in the region by construction, since no particle lies there; this holds for hard systematic and soft multinomial resampling. Any inter-modal mass under DSR or EROT is therefore attributable to the relaxation. Table~\ref{tab:intermodal_mass} reports the results.

\begin{table}[ht]
\centering
\caption{Inter-modal mass (resampled mass in $|x_0| < 3$) for a static two-component Gaussian as a function of dimension $d$. Lower is better; $0$ means no mass placed between the modes. Results show mean $\pm$ std over $10^4$ simulations.}
\label{tab:intermodal_mass}
\small
\begin{tabular}{lccccc}
\toprule
Method & $d = 1$ & $d = 2$ & $d = 3$ & $d = 5$ & $d = 10$ \\
\midrule
Hard SR & $0.000{\scriptstyle\pm 0.000}$ & $0.000{\scriptstyle\pm 0.000}$ & $0.000{\scriptstyle\pm 0.000}$ & $0.000{\scriptstyle\pm 0.000}$ & $0.000{\scriptstyle\pm 0.000}$ \\
Soft & $0.000{\scriptstyle\pm 0.000}$ & $0.000{\scriptstyle\pm 0.000}$ & $0.000{\scriptstyle\pm 0.000}$ & $0.000{\scriptstyle\pm 0.000}$ & $0.000{\scriptstyle\pm 0.000}$ \\
EROT & $0.000{\scriptstyle\pm 0.000}$ & $0.000{\scriptstyle\pm 0.000}$ & $0.747{\scriptstyle\pm 0.052}$ & $1.000{\scriptstyle\pm 0.000}$ & $1.000{\scriptstyle\pm 0.000}$ \\
DSR (Morton) & $0.287{\scriptstyle\pm 0.005}$ & $0.283{\scriptstyle\pm 0.005}$ & $0.282{\scriptstyle\pm 0.005}$ & $0.282{\scriptstyle\pm 0.005}$ & $0.282{\scriptstyle\pm 0.005}$ \\
DSR (w/o Morton) & $1.000{\scriptstyle\pm 0.000}$ & $1.000{\scriptstyle\pm 0.000}$ & $1.000{\scriptstyle\pm 0.000}$ & $1.000{\scriptstyle\pm 0.000}$ & $1.000{\scriptstyle\pm 0.000}$ \\
\bottomrule
\end{tabular}
\end{table}

The results support the use of Morton sorting in this setting. Without Morton sorting, every resampled particle lands between the modes; with Morton sorting the inter-modal mass drops to $\approx 0.28$, and it does not degrade over this range of $d$. EROT, by contrast, is sensitive to dimension: for $d \ge 3$ it places most resampled mass between the modes. Morton sorting thus substantially reduces inter-modal mass relative to no sorting for dimensions up to $10$. Multimodal posteriors remain challenging for particle filters in general~\cite{vermaak2023mixturepf,thrun2001robustmc}, and investigating differentiable resampling under such posteriors is an interesting direction for future work.

\section{Additional Linear Gaussian Experiments}\label{app:linear_gaussian_extra}
The linear Gaussian SSM of Section~\ref{sec:experiments}, as the state transition model is known, admits optimal filtering recursion via the Kalman filter, which lets us approach two questions not accessible in the chaotic or real-world settings: how the marginal likelihood at evaluation time decomposes across methods, and how the DSR gradient compares to an unbiased VSMC gradient.

\subsection{Marginal likelihood at evaluation time}\label{app:marginal_likelihood}
At evaluation time all methods use hard systematic resampling, so $\log \hat{p}(\mathbf{y}_{1:T})$ is directly comparable across methods, and on the linear Gaussian model the exact log-evidence is available from the Kalman filter. This yields the decomposition
\begin{equation}
\underbrace{\log p_{\mathrm{true}}(\mathbf{y}_{1:T}) - \log p_{\hat{\pmb\phi}}(\mathbf{y}_{1:T})}_{\text{model error}}
\;+\;
\underbrace{\log p_{\hat{\pmb\phi}}(\mathbf{y}_{1:T}) - \mathbb{E}[\log \hat{p}_N(\mathbf{y}_{1:T})]}_{\text{Jensen's gap}},
\end{equation}
where the model error measures how close the learned transition model is to the data-generating model, and the Jensen's gap measures how well an $N$-particle filter approximates the evidence of that learned model. Using the trained models from Section~\ref{sec:experiments} on the linear Gaussian SSM with default parameters ($T = 100$, SMNR $= 10$~dB, mean $\pm$ std over model seeds on $10$ shared sequences), we obtain Table~\ref{tab:marginal_likelihood}.

\begin{table}[ht]
\centering
\caption{Evaluation-time evidence decomposition on the linear Gaussian SSM ($T = 100$, SMNR $= 10$~dB). The true log-evidence is $-232.367$. Results show mean $\pm$ std over model seeds.}
\label{tab:marginal_likelihood}
\small
\begin{tabular}{lcccc}
\toprule
Method & $\log p_{\hat{\pmb\phi}}$ & $\mathbb{E}[\log \hat{p}_N]$ & model error & Jensen's gap \\
\midrule
DSR  & $-232.942{\scriptstyle\pm 0.087}$ & $\mathbf{-235.138}{\scriptstyle\pm 0.385}$ & $0.575{\scriptstyle\pm 0.087}$ & $\mathbf{2.196}{\scriptstyle\pm 0.302}$ \\
EROT & $-233.065{\scriptstyle\pm 0.104}$ & $-235.851{\scriptstyle\pm 0.414}$ & $0.699{\scriptstyle\pm 0.104}$ & $2.786{\scriptstyle\pm 0.467}$ \\
Hard & $\mathbf{-232.795}{\scriptstyle\pm 0.038}$ & $-236.147{\scriptstyle\pm 0.924}$ & $\mathbf{0.428}{\scriptstyle\pm 0.038}$ & $3.352{\scriptstyle\pm 0.901}$ \\
Soft & $-232.996{\scriptstyle\pm 0.138}$ & $-237.215{\scriptstyle\pm 0.749}$ & $0.630{\scriptstyle\pm 0.138}$ & $4.218{\scriptstyle\pm 0.669}$ \\
\bottomrule
\end{tabular}
\end{table}

All four methods are within $0.7$ of the true evidence. Hard SR attains the smallest model error ($0.428$ versus $0.575$ for DSR, a gap outside the spread over seeds), but the methods separate on the Jensen's gap, where DSR attains $2.196$ against $3.352$ for hard SR and with roughly a third of the spread ($\pm .302$ versus $\pm .901$). The two columns measure different things: since the optimised quantity is $\mathbb{E}[\log \hat{p}_N]$ rather than $\log p$, these results indicate that DSR learns parameters whose evidence the finite-particle filter approximates more accurately.

\subsection{Comparison against unbiased VSMC gradients}\label{app:unbiased_gradient}
We compare the DSR gradient against an unbiased VSMC gradient. Here DSR is the plain relaxation of Eq.~\ref{eq:soft_assignment}, without Morton sorting or the gradient clip. This is to study the gradient bias induced by DSR. We follow the linear Gaussian experiment of \cite{naesseth2018variational} adapted to our setting. We use the linear Gaussian SSM introduced in Sec.~\ref{sec:linear_gaussian_experiment}, using SMNR $= 10$~dB with $N = 4$, $T = 4$, and additionally a two-dimensional version (a single rotation block) at the same $N$ and $T$. For each we generate one observation sequence $\mathbf{y}_{1:T}$ that every run below targets. Gradients are taken with respect to all $d^2$ entries of $\mathbf{A}$ and the $d$ diagonal entries of $\mathbf{Q} = \mathrm{diag}(\mathrm{softplus}(\pmb\ell))$, evaluated at the data-generating parameters. The unbiased gradient $\nabla\mathcal{L}$ is estimated by averaging over $1000$ bootstrap noise draws and exhaustively summing over ancestor vectors (the vectors are the systematic combs induced by a single uniform draw). Let $\hat{\mathbf{g}}$ denote an estimate obtained by averaging $n = 1600 \, (= 100/4 \times 64)$ independent SMC runs, matching the number of observations consumed in our training configuration (Appendix~\ref{sec:training_arch_config}), each run random through its bootstrap noise and resampling offsets. We report
\begin{equation}
\mathrm{Bias} = \big\|\mathbb{E}[\hat{\mathbf{g}}] - \nabla\mathcal{L}\big\|, \quad
\mathrm{Variance} = \mathbb{E}\big\|\hat{\mathbf{g}} - \mathbb{E}[\hat{\mathbf{g}}]\big\|^2, \quad
\mathrm{MSE} = \mathbb{E}\big\|\hat{\mathbf{g}} - \nabla\mathcal{L}\big\|^2 = \mathrm{Bias}^2 + \mathrm{Variance}.
\end{equation}
We estimate $\mathbb{E}[\hat{\mathbf{g}}]$ from the same $1000$ bootstrap draws with $500$ offset sequences ($5 \times 10^5$ runs), and the variance separately from $200$ independent realisations of $\hat{\mathbf{g}}$ at $n = 1600$. Results are in Table~\ref{tab:unbiased_gradient}.

\begin{table}[ht]
\centering
\caption{Gradient bias, variance, and MSE against the unbiased VSMC gradient on the linear Gaussian SSM. Hard SR drops the score-function term. Results show mean $\pm$ standard errors.}
\label{tab:unbiased_gradient}
\small
\begin{tabular}{lccc}
\toprule
Method & Bias & Variance & MSE \\
\midrule
\multicolumn{4}{c}{(a) $d = 2$, $N = 4$, $T = 4$; \ $\|\nabla\mathcal{L}\| = 9.64$} \\
\midrule
Hard SR & $2.324{\scriptstyle\pm 0.014}$ & $0.308{\scriptstyle\pm 0.016}$ & $5.707{\scriptstyle\pm 0.067}$ \\
DSR ($\tau = 1.0$) & $1.256{\scriptstyle\pm 0.000}$ & $0.241{\scriptstyle\pm 0.014}$ & $1.818{\scriptstyle\pm 0.014}$ \\
DSR ($\tau = 0.5$) & $1.257{\scriptstyle\pm 0.001}$ & $0.245{\scriptstyle\pm 0.014}$ & $1.825{\scriptstyle\pm 0.014}$ \\
DSR ($\tau = 0.1$) & $0.764{\scriptstyle\pm 0.020}$ & $0.389{\scriptstyle\pm 0.031}$ & $0.973{\scriptstyle\pm 0.043}$ \\
DSR ($\tau = 0.05$) & $0.370{\scriptstyle\pm 0.050}$ & $0.723{\scriptstyle\pm 0.037}$ & $0.860{\scriptstyle\pm 0.052}$ \\
DSR ($\tau = 0.01$) & $0.343{\scriptstyle\pm 0.227}$ & $13.409{\scriptstyle\pm 2.583}$ & $13.526{\scriptstyle\pm 2.587}$ \\
\midrule
\multicolumn{4}{c}{(b) $d = 4$, $N = 4$, $T = 4$; \ $\|\nabla\mathcal{L}\| = 20.50$} \\
\midrule
Hard SR & $3.627{\scriptstyle\pm 0.028}$ & $2.092{\scriptstyle\pm 0.058}$ & $15.248{\scriptstyle\pm 0.212}$ \\
DSR ($\tau = 1.0$) & $2.869{\scriptstyle\pm 0.001}$ & $2.073{\scriptstyle\pm 0.057}$ & $10.304{\scriptstyle\pm 0.057}$ \\
DSR ($\tau = 0.5$) & $2.914{\scriptstyle\pm 0.005}$ & $2.127{\scriptstyle\pm 0.058}$ & $10.615{\scriptstyle\pm 0.064}$ \\
DSR ($\tau = 0.1$) & $2.131{\scriptstyle\pm 0.088}$ & $4.895{\scriptstyle\pm 0.192}$ & $9.437{\scriptstyle\pm 0.420}$ \\
DSR ($\tau = 0.05$) & $1.720{\scriptstyle\pm 0.190}$ & $16.382{\scriptstyle\pm 1.864}$ & $19.340{\scriptstyle\pm 1.975}$ \\
DSR ($\tau = 0.01$) & $1.412{\scriptstyle\pm 1.216}$ & $384.805{\scriptstyle\pm 102.564}$ & $386.799{\scriptstyle\pm 102.621}$ \\
\bottomrule
\end{tabular}
\end{table}

DSR has a smaller gradient bias than hard SR (which drops $g_{\mathrm{score}}$) at every temperature and in both settings, and the bias decreases as $\tau$ decreases: at $\tau = 1.0$ the bias drops from $2.324$ to $1.256$ ($d = 2$) and from $3.627$ to $2.869$ ($d = 4$) with essentially unchanged variance. Decreasing $\tau$ further keeps reducing the bias but increases the variance sharply. At the operating temperature $\tau = 0.1$, DSR is competitive on MSE with hard SR and the other $\tau$ values, attaining the lowest MSE at $d = 4$ and the second lowest at $d = 2$. This experiment makes the bias--variance tradeoff explicit and motivates the annealing strategy studied in Appendix~\ref{app:annealing}.



\end{document}